\documentclass[preprint,12pt]{elsarticle}

\usepackage{algorithm2e}
\usepackage{amssymb}
\usepackage{amsmath}
\usepackage{dsfont}

\usepackage{subcaption}
\usepackage{xcolor}

\usepackage{bbm}

\newcommand{\from}{\leftarrow}

\usepackage{amsthm}
\newtheorem{prop}{Proposition}

\journal{Pattern Recognition}

\begin{document}

\begin{frontmatter}



\title{A probabilistic framework for handwritten text line segmentation}


\author{Francisco~Cruz}
\ead{fcruz@cvc.uab.es}
\author{Oriol~Ramos~Terrades}
\ead{oriolrt@cvc.uab.es}

\address{Computer Vision Center-Dept. Ci\`encies de la Computaci\'o, Universitat Aut\`onoma de Barcelona, Edifici O, 08193, Bellaterra (Barcelona), Spain}

\begin{abstract}
We successfully combine Expectation-Maximization algorithm and variational approaches for parameter learning and computing inference on Markov random fields. This is a general method that can be applied to many computer vision tasks. In this paper, we apply it to handwritten text line segmentation. We conduct several experiments that demonstrate that our method deal with common issues of this task, such as complex document layout or non-latin scripts. The obtained results prove that our method achieve state-of-the-art performance on different benchmark datasets without any particular fine tuning step.
\end{abstract}


\end{frontmatter}


\section{Introduction}

The task of text line segmentation arises as a particular case of physical layout analysis where the entities to segment are text lines of a text region.
Its importance in the document analysis field relies on the fact that many other tasks, as word-spotting or handwritten recognition, depend on the text line segmentation results. 
The problem of detecting text lines was stated decades ago in the context of machine-printed text \cite{Nagy2000}. Since then, many methods have been proposed with remarkable results to the point of be considered as a solved problem for machine-printed text \cite{OGorman1993, Liang1999, Nagy1992, Plamondon2000}.
Printed text lines are expected to be uniform throughout the document, as well as to be free of line overlapping and warping effects. However, if these conditions are not satisfied, these methods can fail.

The segmentation of freestyle handwritten documents is still a challenging problem. The large variability in writing styles and possible document layouts generates a set of challenges to overcome.
First, text line orientation can vary along the document or within the same paragraph. Besides, it is also possible to find curved or broken text lines result of the writer style. 
Second, text lines can overlap with each other. This is produced by the contact between ascenders and descenders of characters or just because of cramped text. This effect is a problem for many methods, which expect certain separation between lines.
Third, and regarding the document layout, text can be located in any part of the document. For instance, text in letters is usually located at the center of the document. However, handwritten annotations in administrative documents text is located randomly at any document location.
Many of the methods, which has recently been proposed in the last years, focus on particular kind of document collections. Other methods focus on specific problems, such as touching lines or curved lines and others are tailored to particular scripts or document layouts that make them hard to generalise to other collections \cite{Sulem2007}.

Statistical approaches are less commonly applied for this task and they are often limited to model local features or for post-processing tasks. Markov Random Fields (MRF) have proved to be a good choice for many computer vision tasks, since they provide a strong statistical framework to model prior information about the problem and the relationships between the set of variables \cite{Wang2013}. 
However, inference and parameter learning are intractable for certain model topologies with a large number of variables and high-order relationships. In these cases approximate methods are required to efficiently learn model parameters and perform inference tasks~\cite{Komodakis2015,Liu2015,Schwing2016}.

In this paper we propose a general method for handwritten text line segmentation based on the estimation of a set of regression lines. We successfully combine Expectation-Maximization (EM) algorithm and variational approaches for parameter learning and inference on the model. Thus, we summarize the main contributions of this paper as follows: 

\begin{enumerate}
\item It is a general method devised to be script, layout, and language independent. Besides, it can be applied on documents with complex  layouts. 
\item It can easily extended with any prior knowledge of the task by the inclusion of new feature functions. 
\item It performs parameter learning in an algorithm that combines MRF parameter learning within an EM process.
\end{enumerate}

The rest of the paper is organized as follows: 
In Section 2 we review some of the the main works and techniques proposed for the handwritten line segmentation task. 
In Section 3 we describe the proposed model and learning algorithm. 
In Section 4 we describe the initialization and post-process steps. 
In Section 5 we describe an exhaustive evaluation and the obtained results. Finally, in Section 6 we present the conclusions of this work.

\section{Related Work}   

In the last years there have been many attempts to tackle the task of text line segmentation from different perspectives. The variety of methods promoted the celebration of several contests and benchmark datasets \cite{Gatos2007, Gatos2009, Gatos2013}. As a particular case of physical layout analysis, common approaches are based on the bottom-up and top-down paradigms. However, hybrid approaches have emerged using a wide range of techniques.

Bottom-up approaches are based on the analysis at pixel level or at connected component level. These methods group pixel, or CC, first into characters, then into words and ultimately, to lines. These methods usually obtain good results when exists a clear separation between lines and characters [19,20,21]. However, in conditions of crowded text it may result in text line overlapping. In some cases, these methods are complemented with a post-process step where the overlapping is detected and treated apart \cite{Li2008}.
Different works usually differ in the grouping mechanism. Geometric relationships as distance, angle, or similarity are common criteria \cite{Simon1997,Jaeger2006}. Clustering methods \cite{Yin2009}, or the optimization of a fitting function \cite{Koo2012} have been also proposed. In \cite{Li2008} the level set method is used in combination with a probabilistic function to find line boundaries. 

Top-down approaches analyze top level entities as text blocks, and split them into lines and words, consecutively. 
Projection profile-based methods are the most representative of this type \cite{Nagy1992}. The idea is to project text pixels on the vertical axis and analyze the resulting histogram. Maximum and minimum peaks shall represent, in an ideal case, the location of the text lines and line spacing, respectively \cite{Manmatha2005}. 
The sensitivity to orientation changes or curved lines is usually tackled dividing the document in vertical strips and process separately \cite{Bruzzone1999}. The results on each of the strips are then aligned by means of geometrical properties \cite{kavallieratou2002, Pal2004}, or probabilistic features \cite{Arivazhagan07,Papavassiliou2010}.
In addition, it is common to use common top-down approaches to find an initial text line location, and then run another more sophisticated method to find them \cite{Shi2009}. 
These methods usually fail on freestyle handwritten documents where text is randomly spread over the whole document, or text lines have a high overlapping degree or curvature.

Hybrid methods combine bottom-up and top-down methodologies with other techniques.
The Hough Transform is used to locate text lines by extracting a set of key points of the image and computing the lines that best fit these set of points. These lines are then combined according to different criteria as contextual information \cite{Fletcher1988} or an exhaustive search approach \cite{Likforman1995}. 
In general, Hough-based methods are highly affected by touching text lines and crowed text \cite{Louloudis2008, Pu1998, Shi2009}.
Morphology-based operators have also produced good results \cite{roy2008, Nicolaou2009, Alaei2011}. These methods analyze morphological properties of the documents to infer text line location. The run-length smearing algorithm (RLSA), is a representative example of this approach \cite{Wong1982, Shi2004}. These methods obtain good results on skewed and curved lines. However, touching text lines still affects negatively to the performance.
Graph-based approaches, where lines are represented by minimum cost paths, and active contours (snakes) are other examples of methodologies applied \cite{Fernandez2014, Kumar2011, Liwicki2007, Bukhari2009, Bukhari2009b, Bukhari2013}.

The use of probabilistic graphical models have been mainly focused in the task of document segmentation and text extraction \cite{Nicolas2007}. There, a MRF is defined according to the grid-like structure of the pixels considering pairwise relationships between neighbors. The main challenge relies on the inference process. The computation of exact inference is an NP-hard problem in general, and it becomes intractable for most of loopy MRF configurations. Approximate algorithms as belief propagation \cite{Pearl1982} and its extensions like the Generalized Belief Propogation (GBP) have been widely used for many segmentation tasks. However, these algorithms do not always guarantee to converge. Variational methods based on the minimization of different kind of convex  free energies~\cite{Heskes2006} provide convergent extensions of the GBP algorithm~\cite{Yedidia05}. However, the convergence rate of these methods is still low and can not be applied, in practice to models with high-order cliques. Some approaches take advantage of distributed architectures to speed up learning and inference tasks~\cite{Schwing2016}. More recently, it has increased relevance weighted mini-bucket (WMB) methods as a trade-off between inference accuracy and time complexity~\cite{Dechter2003,Liu2011,Flerova2016}. Hybrids methods, which combines sampling-based methods like importance sampling (IS) and variational methods has also been developed to increase both the accuracy and the efficiency of both inference and parameter learning \cite{Liu2015}. However, there are still room for improvement in both inference and learning methods for MRF models.

\section{Model}
\label{sec:model}

In this section we describe the model proposed for the task of handwritten text line segmentation. For a given text line, our hypothesis is that, if we know the set of pixels that compose it, we can estimate a regression line through these pixels that is a good estimate of the original line position. Besides, each of these pixels will have a higher probability to be assigned to this line than to another.

We select a random set of $N$ text pixels ensuring an uniform distribution along the document image in order to cover all the textual components. The use of a random sample reduce the complexity of the overall method, and according to previous works it does not significantly affects to the final result as long as the sample covers all the data~\cite{Cruz2013}. 

We define a MRF model composed of two kind of random variables. On the one hand, we have random variables $e=(x,y)$ which correspond to pixel coordinates and, on the other hand, we have hidden variables, $h$, which denote the labels of text lines. The topology of our model is given by the Delaunay triangulation computed from the set of random pixels, as we show in~\figurename~\ref{fig:crfwim}. The result is an undirected graph $\mathcal{G} = (\mathcal{V},\mathcal{E})$ where vertexes in $\mathcal{V}$ are the variables $h$ and $e$. The set $\mathcal{E}$ is composed of two kind of edges. First, we have edges between pixel coordinates $e$ and the corresponding text line label. Second, we have edges between adjacent hidden variables $h$.

\begin{figure}[ht]
\centering
\begin{subfigure}{.45\textwidth}
  \centering
  \includegraphics[width=.7\linewidth]{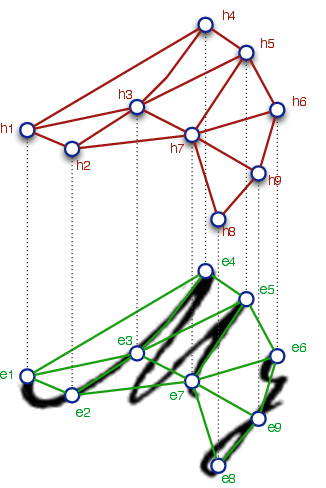}
  \caption{Undirected graph: $\mathcal{G}=(\mathcal{V},\mathcal{E})$}
  \label{fig:crfwim}
\end{subfigure}
\begin{subfigure}{.45\textwidth}
  \centering
  \includegraphics[width=.6\linewidth]{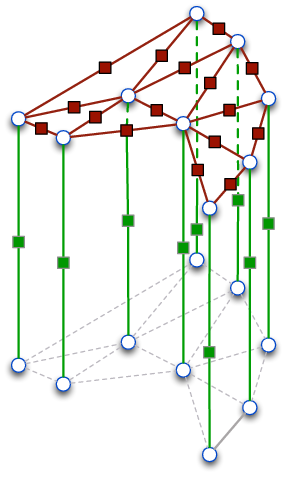}
  \caption{Factor graph}
  \label{fig:crffactors}
\end{subfigure}
\caption{Illustration of a region of the proposed MRF. (a) Variables in green represent the {\em observed} pixels, $e$. In red, hidden variables $h$ representing the text line labels. (b) Illustration of the two types of factors. Green factors are the $v$ factors that relates the observed and the hidden values. Red factors are the $u$ factors composed only by the hidden values.}
\label{fig:CRF}
\end{figure}

We represent our MRF model by a factor graph composed of two type of factor functions in agreement with the two kind of edges describe above, see~\figurename~\ref{fig:crffactors}. 
First, we have factor functions modeling dependencies between observed pixels, $e$, and hidden variables, $h$. These are 3-order factors since pixel coordinates are two random variables and we denote them by $\Psi_v$, with $v\in[1,\,N]$. 
Second, we have factor functions modeling dependencies between pairs of hidden variables and we denote them by $\Psi_u$, where $u=\{i,j\}$ runs over the edges of the Delaunay triangulation. Thus, the MRF factorizes as a product of $\Psi_u$ and $\Psi_v$ as follows: %
\begin{equation}
\label{eq:ProbDistr}
p(e,h|\Theta) = \frac{1}{Z(\Theta)} \prod_{v} \Psi_{v} (e_{v},h_{v}|\Theta_a) \prod_{u} \Psi_{u}(h_{u}|\Theta_b)  = \prod_v p_v(e_v|h_v,\Theta_a) p(h|\Theta_b)
\end{equation}

\noindent where $\Theta=( \Theta_a,\Theta_b)$ is the set of shared parameters, i.e. all factors $\Psi_v$ share the same parameters $\Theta_a$, and similarly, all factors $\Psi_u$ share parameters $\Theta_b$.  Note that the topology of $\mathcal{G}$ allow us to factorizes the MRF model as a product of conditional likelihood probabilities $p_v(e_v|h_v,\Theta_a)$ of pixels $e_v = (x_v,y_v)$  and the {\em prior} probability of hidden variables $h$, $p(h|\Theta_b)$. 

Our method relies on the classic EM algorithm~\cite{Dempster1977}. This algorithm is based on the definition of a function $Q$, which is the conditional expectation of the likelihood function of a probability density function:%
\begin{equation}
\label{EqQ}
Q(\Theta|\Theta') = \mathrm{E}_{h}(\log p(h,e|\Theta)|e,\Theta')
\end{equation}

\noindent thus, in the Expectation (E) step, $Q$ is evaluated given the current set of parameters $\Theta'$. Then, in the Maximization (M) step, new parameters $\Theta$ are computed. These new parameters are obtained by computing the partial derivatives of $Q$ with respect to each single model parameter $\theta_k \in \Theta$. This scheme is repeated until both sets of parameters: $\Theta'$ and $\Theta$ are equal.

Our method essentially follows the same scheme. The main difference concerns the parameter learning step of the MRF model. First, in the E-step, we update the parameters of the {\em prior} probability $p(h|\Theta_b)$. We update these parameters using the proposed extension of the GBP, which we explain in section~\ref{sec:GMP}, to allow parameter learning. 
With the parameters learned we can approximate the posterior probability of each single hidden variable $h_v$ given the coordinates $e_v$. Then, in the M-step, we update the parameters $\Theta_a$, which correspond to the regression lines. In summary, our proposed scheme is Algorithm~\ref{alg:EM}:

\begin{algorithm}
\begin{enumerate}
\item $\Theta=(\Theta_a,\Theta_b)$ initialization
\item E-step: parameter learning of {\em prior} probability
	\begin{enumerate}
	\item Update $\Theta_b$: $\Theta_b \leftarrow \Theta_b'$
	\item Estimate $p(h_v|e_v,\Theta_a',\Theta_b)$
	\end{enumerate}
\item M-step: estimation of regression lines
\begin{enumerate}
\item Update $\Theta_a$: $\Theta_a \leftarrow \Theta_a'$
\end{enumerate} 
\item Repeat steps 2-3 until convergence
\item End
\end{enumerate}
\caption{EM algorithm for MRF models}
\label{alg:EM}
\end{algorithm}

In the remainder of this section we explain the linear regression scheme and how to estimate the new updates of its parameters $\Theta_a$. Then we explain how to learn model parameters linked to the prior probability $p(h|\Theta_b)$. We will conclude this section with the definition of the feature functions used for the handwritten text line segmentation task.

\subsection{EM algorithm for linear regression}
\label{sec:EM}

We defined a set of factor functions that encode the information within the MRF. Each factor function is composed of a set of feature functions $f_k$ and $g_k$ where $k$ runs in $I_{u}$ or $I_{v}$ depending whether the feature function is defined on $\Psi_u$ or $\Psi_v$, respectively. These feature functions are embedded in factors as: %
\begin{equation}
\label{ff}
\begin{split}
\log \Psi_{u} &= \sum_{k \in I_{u}}  f_k(h_{u} | \Theta_b) \\
\log \Psi_{v} &= \sum_{k \in I_{v}}  g_k(h_{v},e_{v}| \Theta_a)
\end{split}
\end{equation}
\noindent we replace the above definitions and the MRF model of Eq.~\eqref{eq:ProbDistr} in $Q$ and we have:%
\begin{equation}
\label{EqQ2}
\begin{split}
Q(\Theta | \Theta') = \sum_v   \sum_{h_v}  \left[\sum_{k \in I_v}g_k(h_v,e_v | \Theta_a)  -  \log Z_v(h_v,\Theta_a)\right]p_v(h_v|e_v,\Theta_a',\Theta_b')  + \\
+  \sum_u \sum_{k \in I_u} \left[ \sum_{h_u} f_k(h_u | \Theta_b) p_u(h_u|\Theta_b') \right] - \log Z_0(\Theta_b)
\end{split}
\end{equation}

\noindent  where $Z_v(h_v,\Theta_a)$ and  $Z_0(\Theta_b)$ denote, respectively, the partition function of  the conditional likelihood probabilities and the {\em prior} probability. With this expression we find the new parameter updates by finding the local maximum of $Q$, which correspond with the M-step. 

\begin{figure}[t]
\begin{center}

    \includegraphics[width=0.4\paperwidth, height=150pt]{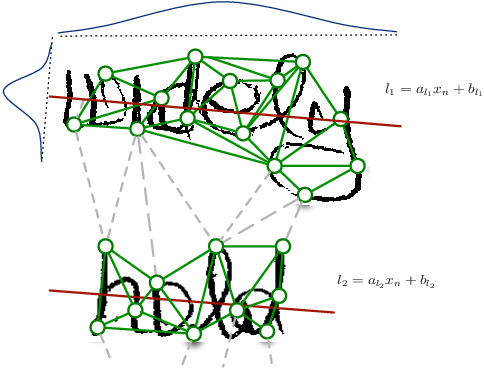}
	\caption{Hypothetical region of our graphical model that relates the pixels from words from consecutive lines. Messages sent through the dashed lines are supposed to favor a different label for each connected pixel. }
	\label{fig:crflines}
\end{center}
\end{figure}

We use a linear regression model to fit the text lines in the document. The goal is to estimate a set of $L$ lines in the form $y_v = a_l x_v + b_l$ with vertical variance $\sigma^2_{l,t}$ from the set of pixels that compose it. Besides, in order to fit the size of text lines we define a pair of bounds that defines a segment of $l$. These bounds are given with respect to the center of the segment $c_l$ by the horizontal variance $\sigma^2_{l,s}$. Therefore, a line $l$ is defined by the following five parameters: $\theta_a = \{a_l,b_l,c_l,\sigma_{l,t},\sigma_{l,s}\}$ that define two Gaussian density functions linked to the horizontal and vertical variances. The associated likelihood probabilities are: %
\begin{equation}\begin{split}
p_t(x_v,y_v | h_v=l,\theta_a) &\propto \exp \left\{ -\frac{(y_v - a_l x_v - b_l)^2}{2 \sigma_{l,t}^2}  \right\} \\
p_s(x_v,y_v | h_v=l,\theta_a) &\propto \exp \left\{ -\frac{(x_v - c_l)^2}{2 \sigma_{l,s}^2}  \right\}
\end{split}\label{Gauss1}
\end{equation}

\noindent for a pixel $e_v=(x_v,y_v)$ and a line $l$. These densities will provide a measure of how well a particular pixel fits a line. \figurename~\ref{fig:crflines} shows an example of a MRF region with two regression lines across two hypothetical words from consecutive text lines $l_1$ and $l_2$. Vertical Gaussian function results perpendicular to the regression line since its purpose is to account for line residues. Horizontal Gaussian in return is defined parallel to the x-axis, since it only controls the line length.

The update equations for each parameter are found by computing the partial derivatives with respect to each parameter of Eq.~\eqref{EqQ2}. The update expressions for $\Theta_a$ are similar than in our previous work \cite{Cruz2013}, although in this case the posterior $p_v(h_v=l|e_v,\Theta_a',\Theta_b)$ is given by the inference algorithm explained later in section~\ref{sec:GMP}. We provide all details of their derivation in the supplementary material of this paper. 
For a given document the number of parameters to estimate is $|\Theta_a| = 5L $. Note that only parameters $\sigma^2_{l,t}$ and $\sigma^2_{l,s}$ appear on the partition function $\log Z_v(h_v,\Theta_a)$: 

\begin{equation}\begin{split}
a_l^{new} &= \frac{\sum_v (x_v - \bar{x}) (y_v - \bar{y}) p_v(h_v=l|e_v,\Theta_a',\Theta_b)}  {\sum_v (x_v - \bar{x})^2 p_v(h_v=l|e_v,\Theta_a',\Theta_b)} \\
b_l^{new} &= \frac{\sum_{v} (y_v - a_l^{new} x_v) p_v(h_v=l|e_v,\Theta_a',\Theta_b)}  {\sum_{v} p_v(h_v=l|e_v,\Theta_a',\Theta_b)} \\
c_l^{new} &= \frac{\sum_{v} x_v p_v(h_v=l|e_v,\Theta_a',\Theta_b)}  {\sum_{v} p_v(h_v=l|e_v,\Theta_a',\Theta_b)} \\
\sigma^2_{l,t} &= \frac{\sum_{v} (y_v - a_l^{new} x_v -b_l^{new})^2 p_v(h_v=l|e_v,\Theta_a',\Theta_b)}  {\sum_{v} p_v(h_v=l|e_v,\Theta_a',\Theta_b)} \\
\sigma^2_{l,s} &= \frac{\sum_{v} (x_v - c_l^{new})^2 p_v(h_v=l|e_v,\Theta_a',\Theta_b)}  {\sum_{v} p_v(h_v=l|e_v,\Theta_a',\Theta_b)} 
\end{split}\label{update1}
\end{equation}

In addition, we also estimate the \textit{prior} probability of each line $l$ given the updated parameters $\Theta_a$ as: %
\begin{equation}
\label{eq:pline}
p_v(h_v=l|\Theta_a,\Theta_b) = \frac{1}{N}{\sum_v p_v(h_v=l|e_v,\Theta_a,\Theta_b)}
\end{equation}

The key point is that posterior probabilities $p_v(h_v=l|e_v,\Theta_a',\Theta_b)$ are unknown and consequently we cannot update the parameters of the regression lines. To overcome this problem, we run an approximate inference algorithm  that allow us to learn MRF parameters and estimate $p_v(h_v=l|e_v,\Theta_a',\Theta_b)$.

\subsection{Inference and Learning} 
\label{sec:GMP}

In the previous section, we described how to estimate the parameters $\Theta_a$ linked to regression lines. However, parameters $\Theta_b$ remain unknown and still have to be learned. 
Many parameter learning methods for MRF models relay on free energy methods. These are variational methods that seek density functions that approximate true marginals by beliefs functions that satisfies a set of constraints. Free energies are quite close to $Q(\Theta,\Theta')$ used within the EM algorithm and defined in Eq.~\eqref{EqQ2}, since both are defined in terms of the Kullback-Leibler divergence (KLD). For instance, the free energy associated to Belief Propagation (BP) algorithm is the Bethe energy as: %
\begin{equation}\begin{split}
F_{Bethe}( p(h|e,\Theta_a',\Theta_b) ) &= \sum_u p_u(h_u|\Theta_b)\log p_u(h_u|\Theta_b) + \\
&+ \sum_{v} n_v p_v(h_v|e_v,\Theta_a',\Theta_b) \log p_v(h_v|e_v,\Theta_a',\Theta_b)
\end{split}
\end{equation}

\noindent where $n_v$ are related to the number of neighbors of $h_v$, and can be negative. In our case, we have to include the information given by the likelihood functions of regression lines. So, we define the free energy as: %
\begin{equation}\begin{split}
F( p(h|e,\Theta_a',\Theta_b) ) &= \sum_u p_u(h_u|\Theta_b)\log p_u(h_u|\Theta_b) + \\
&+ \sum_{v} c_v p_v(h_v|e_v,\Theta_a',\Theta_b) \log p_v(h_v|e_v,\Theta_a',\Theta_b) + \\
&+ \sum_v\sum_{k\in I_v}\sum_{h_v} g_{k}(h_v,e_v|\Theta_a') p_v(h_v|e_v,\Theta_a',\Theta_b)
\end{split}\label{eq:FreeEnergy2}
\end{equation}

\noindent where parameters $c_v>0$ are any positive real value. The approximate marginals and conditional marginals have to satisfy the usual constraints used in message-passing methods. We summarize them in Table~\ref{tab:constraint}. 
First, since $p_u$ and $p_v$ are marginal approximations, they have to be {\em normalized}. 
Second, we have to impose the {\em sum-normalization} constraint between $p_u(h_u|\Theta_b)$ and $p_v(h_v|e_v,\Theta_a',\Theta_b)$ to ensure consistency between marginal estimation. 
Unlike usual message-passing algorithms and to well tie the estimated prior probabilities by the model with the observed data, we impose consistency between prior probability of single variables $h_v$, $p_u(h_v|\Theta_b)$, and posterior probability $p_v(h_v|e_v,\Theta_a',\Theta_b)$. 
Finally, we have to ensure coherence between the observations, encoded in the empirical moments $\mu_k$, and model prediction. 
This last set of constraints is the called {\em moment-matching} constraint and it provides the parameter learning step for the pairwise parameters and global prior probability, Eq.~\eqref{eq:pline}. 
Thus, the minimization of Eq.~\eqref{eq:FreeEnergy2} results on a constrained minimization problem that can be solved by means of Lagrange multipliers.

\begin{table}
\begin{center}
\begin{tabular}{c|c|c}
\hline
Constraint &  Formula & L. Multiplier \\
\hline \hline
\textit{normalization}        & $\sum_{h_u} p_u(h_u|\Theta_b) = 1$ & $\nu_u$ \\
 & $\sum_{h_v} p_v(h_v|e_v,\Theta_a',\Theta_b) = 1$ &  $\nu_v$ \\
\textit{sum-normalization}    &$ \sum_{h_{u \setminus v }} p_u( h_{u}|\Theta_b) = p_v(h_v|e_v,\Theta_a',\Theta_b)$ & $\lambda$\\
\textit{moment-matching}      & $\sum_u f_k(h_u) p_u(h_u|\Theta_b) = \mu_k$ & $\theta_k$  \\
\hline
\end{tabular}
\end{center}
\caption{Set of constraints and its corresponding Lagrange multipliers for the optimization problem of Eq.~\eqref{eq:FreeEnergy2}.}
\label{tab:constraint}
\end{table} 

Algorithm~\ref{alg:MP} is the numerical implementation of block gradient descend method applied to the dual problem obtained from the previous minimization problem.
We provide details of this algorithm in the supplementary material of this paper. Basically, the partial derivative with respect to $\theta_k$ provides the parameter learning equation according to the Armijo conditions. The partial derivative with respect to $\lambda_{v \rightarrow u}(h_v)$ lead to the usual message-passing equations. 
After convergence of the algorithm, we are able to get the final value of $p_v(h_v|e_v,\Theta_a',\Theta_b)$ required for the estimation of the new parameters $\Theta_a$.

\begin{algorithm}
\KwData{ $\{ \mu_k\}$: empirical moments.}
\KwResult{$\Theta_b$, $\Lambda$: model parameters, $\{p_v(h_v|e_v,\Theta_a',\Theta_b), p_u(h_u|\Theta_b)\}$ marginals.}
Initialize: $\theta_k= 0$, $\theta_k\in\Theta_b$, $m_{v \to u}(h_v)=1$;\\
\While{ not converged}
{

\For{ $\forall k \in \{I_u\}$ }{
$$\theta_k \leftarrow \theta_k' + \eta \left(\sum_{h_u} f_k(h_u) p_u(h_u|\Theta_b') - \mu_{k} \right )$$
}

\For{ $\forall v$  }
{
\For{ $\forall u \supset v$  }
{
$$ p_u(h_v|\Theta_b)=\sum_{h_{u \setminus v}}  p_u(h_u|\Theta_b) $$
$$m_{v \from u}(h_u)=\frac{p_u(h_v|\Theta_b)}{m_{v \to u}(h_v)}$$
}

$$p_v(h_v|e_v,\Theta_a',\Theta_b) = \frac{1}{Z_v} \left(e^{\sum_k g_k(h_v,e_v|\Theta_a') }\prod_{u \supset v} m_{v \from u}(h_u) \right )^{\!\!\frac{1}{c_v+A_v}}$$

\For{ $\forall u \supset v$  }
{
$$m_{v \to u}(h_v)   =  \frac{p_v(h_v|e_v,\Theta_a',\Theta_b)}{m_{v \from u}(h_u)}$$

$$p_u(h_u|\Theta_b) = \frac{1}{Z_u}\left(e^{ -\sum_k \theta_kf_k(h_u) }\prod_{v \subset u} m_{v \to u}(h_v)\right)^{\!\!\frac{1}{c_v}}$$
}
}
}
\caption{Message passing algorithm for constrained minimization of free energies problem in Eq.~\eqref{eq:FreeEnergy2}. $Z_{u}$ and $Z_{v}$ are the partition function and $\eta$ is a step length satisfying the Armijo condition.}
\label{alg:MP}
\end{algorithm}

\subsection{Feature functions}
\label{sec:featurefunctions}

In previous sections we defined a general pairwise MRF model adapted to the detection of an unknown number of text lines. This model allow a wide range of unary feature functions to estimate text line position and pairwise feature functions to model text line labels between adjacent pixels. Now we describe the set of feature functions $f_k$ and $g_k$ defined in \eqref{ff} used for the task of handwritten text line segmentation.

\paragraph{\textbf{Local fitting}}

This function uses the information provided by the two Gaussian distributions defined in Eq.~\eqref{Gauss1} with a slight modification inspired by \cite{Koo2012}. It corresponds to a flattened Gaussian distribution on its maximum value. The width of this Gaussian plateau is controlled by a threshold $S_l$, which is computed as in \cite{Koo2012} and it estimates the interline space above and below line $l$. In summary, we define this function as:

\begin{equation}  
g_k(h_v = l, e_v|\Theta_a) \triangleq \left\{ \begin{array}{cclc}
      - \frac{(x - c_l)^2}{2 \sigma_{l,s}^2}   & &  if  & d_l \leq r S_l \\
      - \frac{(y - a_l x - b_l)^2}{2 \sigma_{l,t}^2}   -  \frac{(x - c_l)^2}{2 \sigma_{l,s}^2}   &  &  if  & d_l > r S_l
             \end{array}
   \right.
\label{ff2}
\end{equation}	

\noindent where $d_l$ is the residue of the regression line, and $r \in [0,\,1]$. This flattened procedure slightly modifies the computation of the partition function of likelihood probabilities $p_v(e_v|h_v,\Theta_a',\Theta_b)$ but it still depend only on the variances $\sigma^2$, and therefore, the update equations in Eq.~\eqref{update1} remain valid.

\paragraph{\textbf{Line probability}}

This function integrates the {\em prior} probability computed in Eq.~\eqref{eq:pline} into the learning process described in Algorithm~\ref{alg:MP}. This prior probability can be seen as a moment of the indicator function $[h_v=l]$ and consequently we can learn its associated parameter $\theta_l$. We update the corresponding empirical moment $\mu_k$ with the line probability estimated in each iteration. In our case, for each line the empirical moment is $\mu_l=p_v(h_v=l|\Theta_a,\Theta_b)$. Thus, the function is defined as:  
%
%
\begin{equation}
f_k(h_v=l | \Theta_b) \triangleq \theta_l [h_v = l]
\label{ff3}
\end{equation}


With this function we expect to avoid to assign variables to surplus lines, and reinforce the regression lines with higher probabilities.

\paragraph{\textbf{Pairwise function}}

Pairwise functions encode the probability of assigning a set of labels to neighbor variables. 
In our task we encode in this function some assumptions about the configuration of the lines. For example, in a given document two connected variables are more likely to belong to the same text line, i.e. share the same label, or as much, to consecutive lines. Besides, some documents may have two connected variables from non-consecutive text lines, although they represent a few cases with respect to the most common layouts.
We define our pairwise function according to those three possible scenarios. 
The function is defined on $\Psi_u$, and returns the parameter associated to each possible case:

\begin{equation}
f_k(h_i,h_j|\Theta_b) \triangleq \left\{ \begin{array}{llc}
              \theta_0 &  if  & |h_{i} - h_{j}| = 0 \\
              \theta_1 &  if  & |h_{i} - h_{j}| = 1 \\
              \theta_2 &  if  & |h_{i} - h_{j}| \geq 2
             \end{array}
\right.
\label{ff1}
\end{equation}

\noindent where ${\theta_0,\theta_1, \theta_2}$ are parameters in $\Theta_b$ shared for all pair of hidden variables in $u$ and learned with Algorithm~\ref{alg:MP}. 
The empirical moments $\mu_k$ for this function are learned from the training set by analyzing the frequency of each considered case.

In summary, we have 5L parameters to estimate during the M-Step, and 3+L parameters in $\Theta_b$ to learn.

\section{Initialization and final labeling}

In this section we describe the steps required to configure our method for the task of handwritten text line segmentation. 
First, we define the initialization step which is crucial for the good performance of the EM algorithm. Second, we describe the post-process and final labeling.

\subsection{Initialization}
\label{sec:Initialization}

The initialization of our method for handwritten line segmentation consist of two steps. In the first place we detect the different text regions that compose the document image. Then, for each of them we initialize the parameters of the regression lines.

\paragraph{\textbf{Text region segmentation}} 
Our region segmentation process is based on the segmentation method from \cite{Xiao2003}. 
According to the Delaunay triangulation that defined the MRF structure, we analyze the length of the sides of the triangles in order to find a threshold dependent on the image that identifies the ones that are connecting different regions. Once computed, we remove the ones which longest side is above this value. In this way the different regions are isolated. More details of this process and the computation of the threshold can be found in the referenced paper. This step provides flexibility to our method, since it is able to work in documents with complex layouts by dividing the problem in smaller and simpler ones.
An example of this process is shown in \figurename~\ref{fig:Delaunay}. 

\begin{figure}[t]
\centering \begin{subfigure}{.33\textwidth}
  \centering
  \includegraphics[width=1\linewidth]{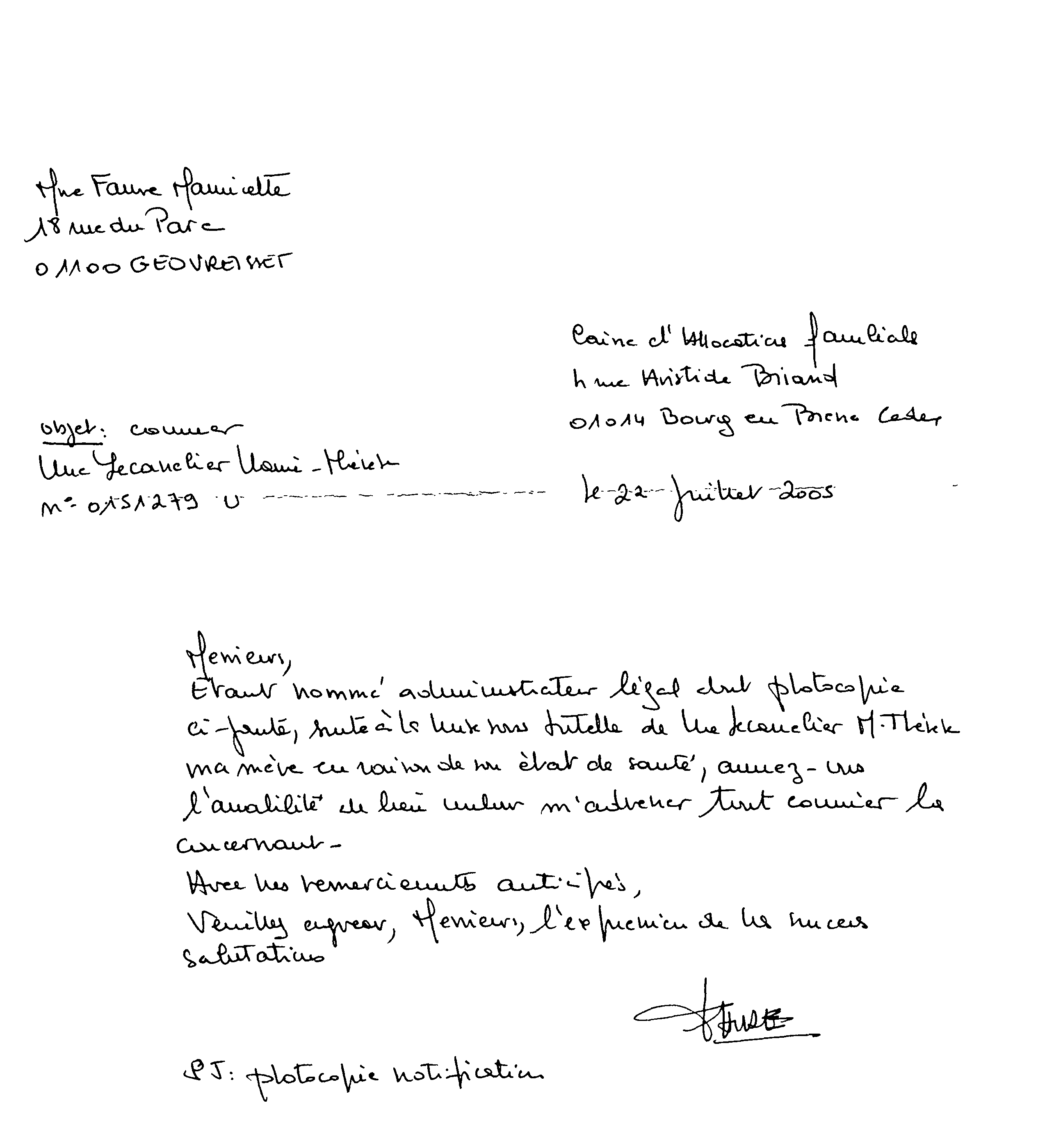}
  \caption{}
\end{subfigure}\begin{subfigure}{.33\textwidth}
  \centering
  \includegraphics[width=1\linewidth]{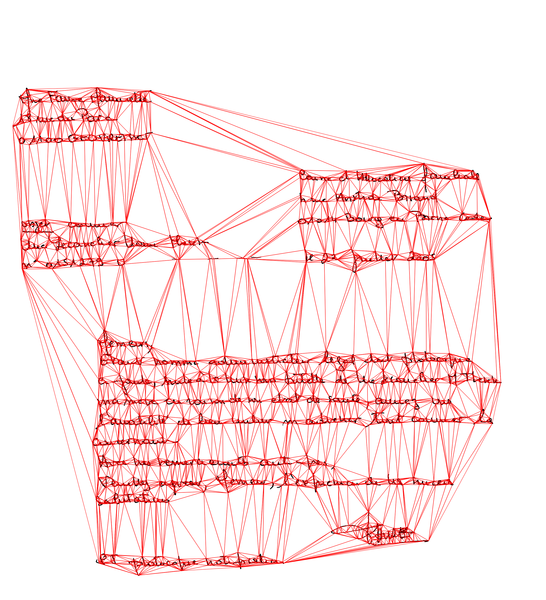}
  \caption{}
\end{subfigure}\begin{subfigure}{.33\textwidth}
  \centering
  \includegraphics[width=1\linewidth]{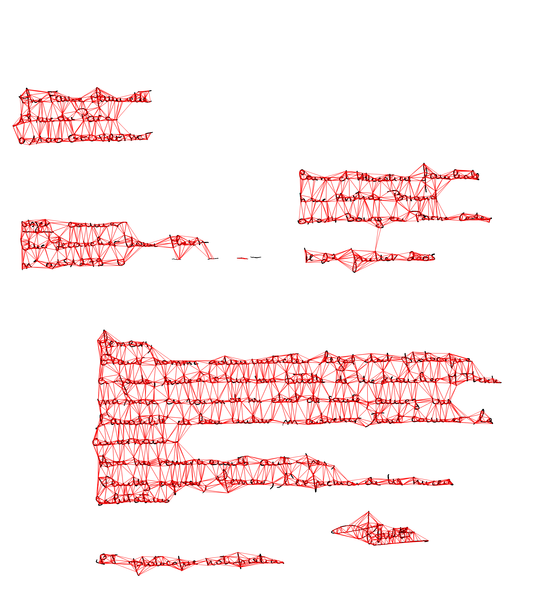}
  \caption{}
\end{subfigure}
\caption{Illustration of the text region segmentation process: a) Original image. b) Delaunay triangulation computed on the set of selected random pixels. c) Result of the process after removing the selected triangles isolating several text regions.}
\label{fig:Delaunay}
\end{figure}

\begin{figure}[t]
\centering
\begin{subfigure}{.45\textwidth}
  \centering
  \includegraphics[width=.8\linewidth]{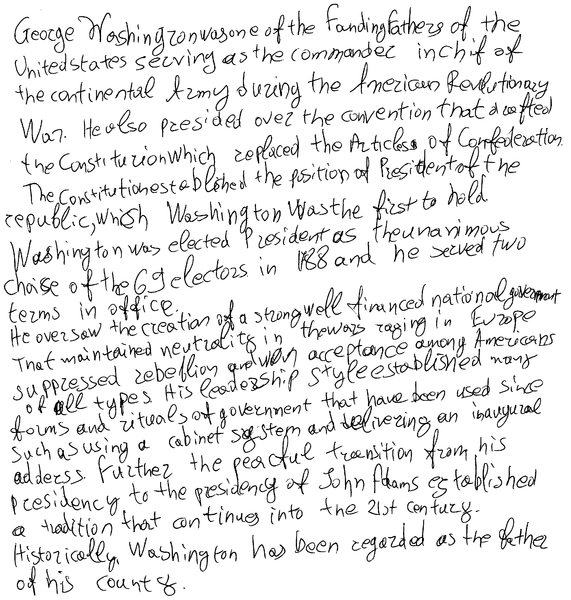}
  \caption{}
\end{subfigure}
\begin{subfigure}{.45\textwidth}
  \centering
  \includegraphics[width=.8\linewidth]{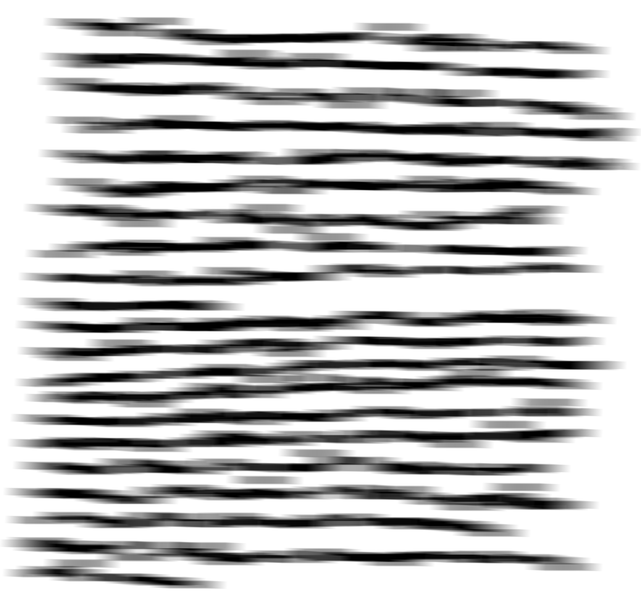}
  \caption{}
\end{subfigure}
\begin{subfigure}{.45\textwidth}
  \centering
  \includegraphics[width=.8\linewidth]{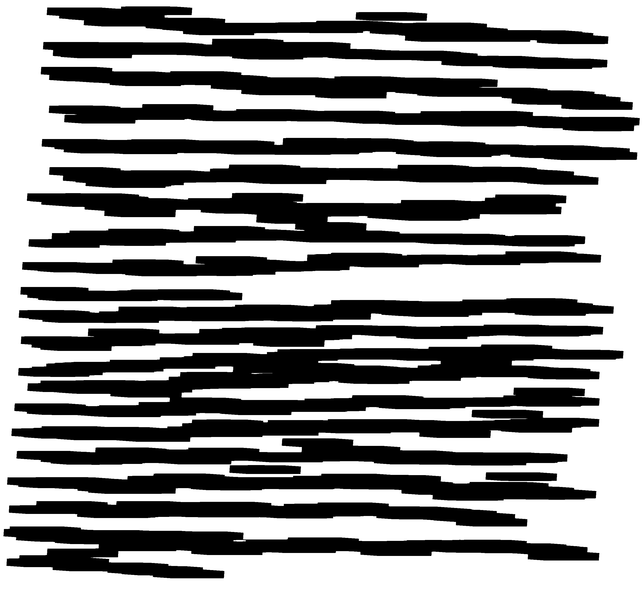}
  \caption{}
\end{subfigure}
\begin{subfigure}{.45\textwidth}
  \centering
  \includegraphics[width=.8\linewidth]{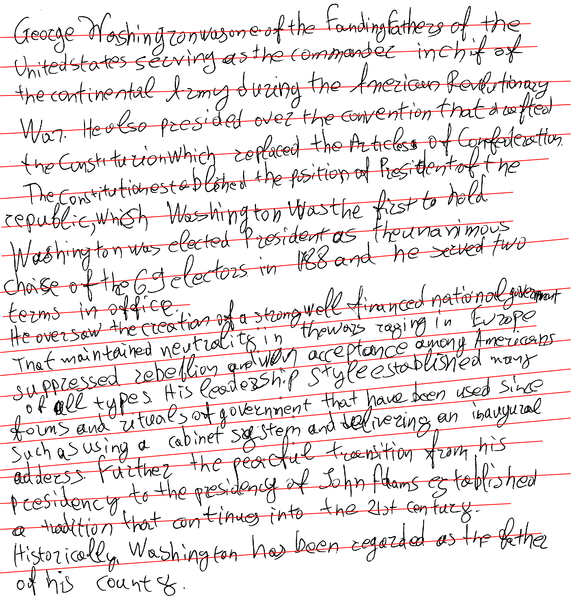}
  \caption{}
\end{subfigure}
\caption{a) Example of a non-accurate initialization process in a document image with crowded text. a) Original image. b) Result of Gaussian filtering. c) Resulting candidate blobs. d) Initial candidate regression lines.}
\label{fig:Initialization}
\end{figure}

\paragraph{\textbf{Initial line hypothesis}} 
It is known that the EM algorithm is often sensitive to the initial choice of parameters. An inaccurate initialization of line parameters may lead the method to fall into a local maximum that do not correspond with the better text line fitting.
We combine several common techniques to propose an initial set of regression lines.

\begin{itemize}

\item Blob estimation: We apply several steps based on the work in \cite{Ziaratban2010} for skew correction and blob identification. We apply a bank of anisotropic 2D Gaussian filters of size $W\times H$
on a range of orientations $\alpha$ and select the one with better response on the projection profile. A similar approach was previously proposed in~\cite{Bukhari2009, Bukhari2009b}. Then, we apply the Otsu binarization method to the filtered image in order to obtain a set of blobs that represent approximate line locations.

\item Overlapping detection: We analyze the obtained blobs in order to detect overlapping as result of touching or curved lines in the document.
To do so, we compute the mean connected component height and divide the blobs proportionally to a threshold $t_o$ of this value. 
Besides, we identify residual blobs result of filtering diacritics or noise components. We compute the ratio of text within each blob and remove the ones under a threshold $t_r$ learned from the training set.

\item Line estimation: The number of resulting blobs define the initial number of candidate lines. For each blob, we estimate the regression line parameters using the common line regression equations on the set of pixels that compose each of them.
\end{itemize}

The initialization step itself could be a good segmentation result in documents with simple layouts where lines are properly separated. In these cases, the execution of our posterior inference process will converge in a few iterations.
However, in complex documents with crowded or slightly curved text the process is more challenging and the initialization usually is not accurate enough, obtaining over-segmented text lines and incorrect initial line locations. \figurename~\ref{fig:Initialization} shows an example of a challenging image where only a few initial lines fit exactly the correct text line. 

A straightforward consequence of the initialization step is the possible over-estimation of lines. 
It is possible that a text line is approximated by two or more initial line segments. Besides, some diacritics from non-romance languages might be also approximated by a short line segments. 
This effect is not a drawback for our method, but the opposite. An initial over-segmentation is recommended, since we need to be sure that we fit the enough number of lines to cover all the text lines. In the case of initializing less than the correct number, some textual components will be probably assigned to the incorrect line, producing several miss detection.

\subsection{Post-process and final labeling}

In the post-process step we analyze the obtained result in order to detect and merge possible fragmented lines and remove surplus ones. After that, we label each of the textual connected components according to the probability given by the MRF model.
 
\paragraph{\textbf{Surplus lines removal}} We remove the extra lines remaining after the algorithm convergence.
Extra lines are featured by a low probability close to zero. We detect and remove these lines by identifying the ones which probability is under an $\varepsilon$ value fixed beforehand.

\paragraph{\textbf{Fragmented lines}} 
The over-segmentation from the initialization step may lead to a fragmentation of a text line. Since our model is linear, the method deals with curved lines by splitting the line into two or more segments. 
We analyze the relative position between the lines in order to identify these cases and unify the fragments into a single line.

\paragraph{\textbf{Final labeling}} 
For each variable $e$ we select the line $l$ that maximizes the probability $p_v(h_v = l|e_v,\Theta)$. We assign the connected component that contains the variable to the line only if all the variables within the component share the same label. Multiple labels in one component usually correspond with touching characters. In that case we label each pixel of the component by distance to the closest regression line.

\section{Experiments}
\label{Exp}

In this section, we describe the experiments performed for the task of handwritten line segmentation. 
We carry out a thorough evaluation on multiple benchmark datasets in order to prove the generality of our method to be applied on documents with different type of layouts and characteristics. 
Besides, we show in an additional experiment the impact of the selection of random pixels for different configurations.

\subsection{Parameters and settings}

Along the previous sections we define a set of parameters that we fix beforehand.
In the initialization step we apply a set of Gaussian filters with orientations in the range $\alpha = [-40,40]$ degrees, and a filter size of $H = \frac{1}{3} H_{cc}$ and $W = 10 W_{cc}$ with a vertical and horizontal standard deviation of $\frac{1}{3} H_{cc}$ and $\frac{10}{3} W_{cc}$, respectively. To identify overlapped and residual blobs we experimentally set $t_o = 2\bar{H}_{cc}$ and $t_r = 0.08$.
 
We fix the ratio $r=0.3$, see Eq.~\eqref{ff2}, and the prior text line probability thhreshold $\varepsilon=10^{-3}$ for extra text line removal. We fix the maximum number of iterations to $50$ and we set the KLD criterion to $K=10^{-4}$. 

We learn pairwise moments $\mu_k$, see Eq.~\eqref{ff1}, from the training set of ICDAR 2013. In addition, we set $c_v =1$. We use these parameter configuration for all the experiments, since they represent an accurate sample of common handwriting script.

\subsection{Metrics}

We report results according to the same metrics used in the ICDAR segmentation contests. The metric is based on counting the number of matches between the detected text lines and the text lines in the ground truth by computing the MatchScote table at pixel level \cite{Phillips1999}. It consist of: Detected lines (M), one-to-one matches (o2o), Detection Rate (DR\%), Recognition Accuracy (RA\%) and F-measure value (FM\%). 
For other datasets on which the ICDAR evaluation tool can not be used we provide results in terms of precision, recall and F-measure computed at pixel level. When possible, we compute Confidence Intervals with confidence value $\alpha=0.05$.

\subsection{Datasets}

We evaluate our method on several benchmark datasets.
On the one hand we evaluate it on the ICDAR 2009 and 2013 handwriting segmentation contest datasets. These datasets contain regular text documents where the text is the main part of the page. In general the documents are free of graphical or non-text elements although some of them may contain small noise.
ICDAR 2009 dataset is composed of 200 test images with 4043 text lines. The documents contain the same extract of text written by several writters in several languages (English, German, Greek and French). 
ICDAR 2013 dataset is an update of the previous one. The dataset contains a set of 150 test images with 2649 text lines also depicted by different writers and in several languages. New features comprise the addition of new more complex languages as Indian Bangla, and new layouts as multi-paragraph and complex skewed and cramped documents. Figure~\ref{fig:ICDAR} shows some examples of documents from this dataset.

\begin{figure} [t]
\centering
\begin{subfigure}{.33\textwidth}
  \centering
  \fbox{\includegraphics[width=.9\linewidth]{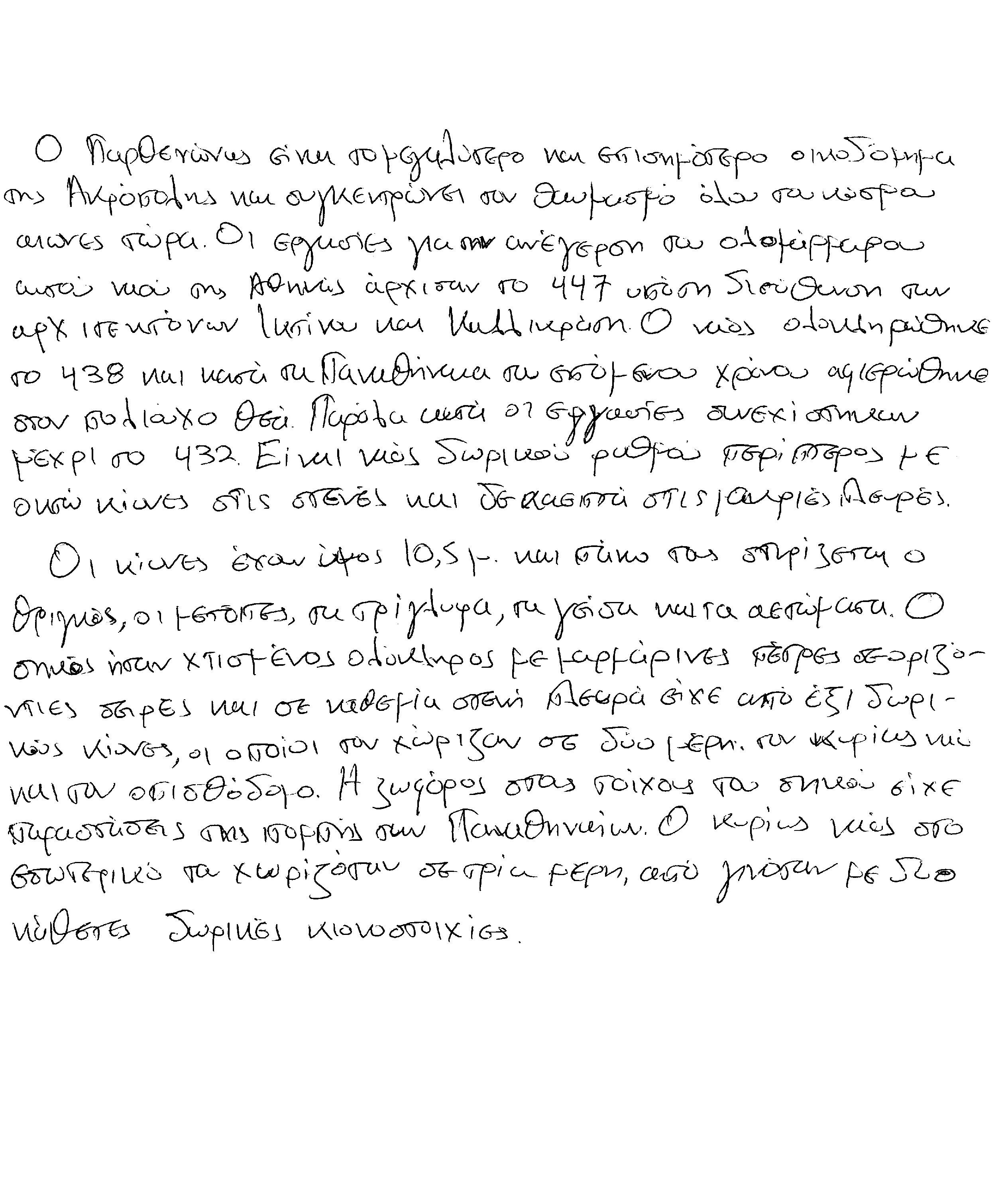}}
  \caption{}
\end{subfigure}%
\begin{subfigure}{.33\textwidth}
  \centering
  \fbox{\includegraphics[width=.9\linewidth]{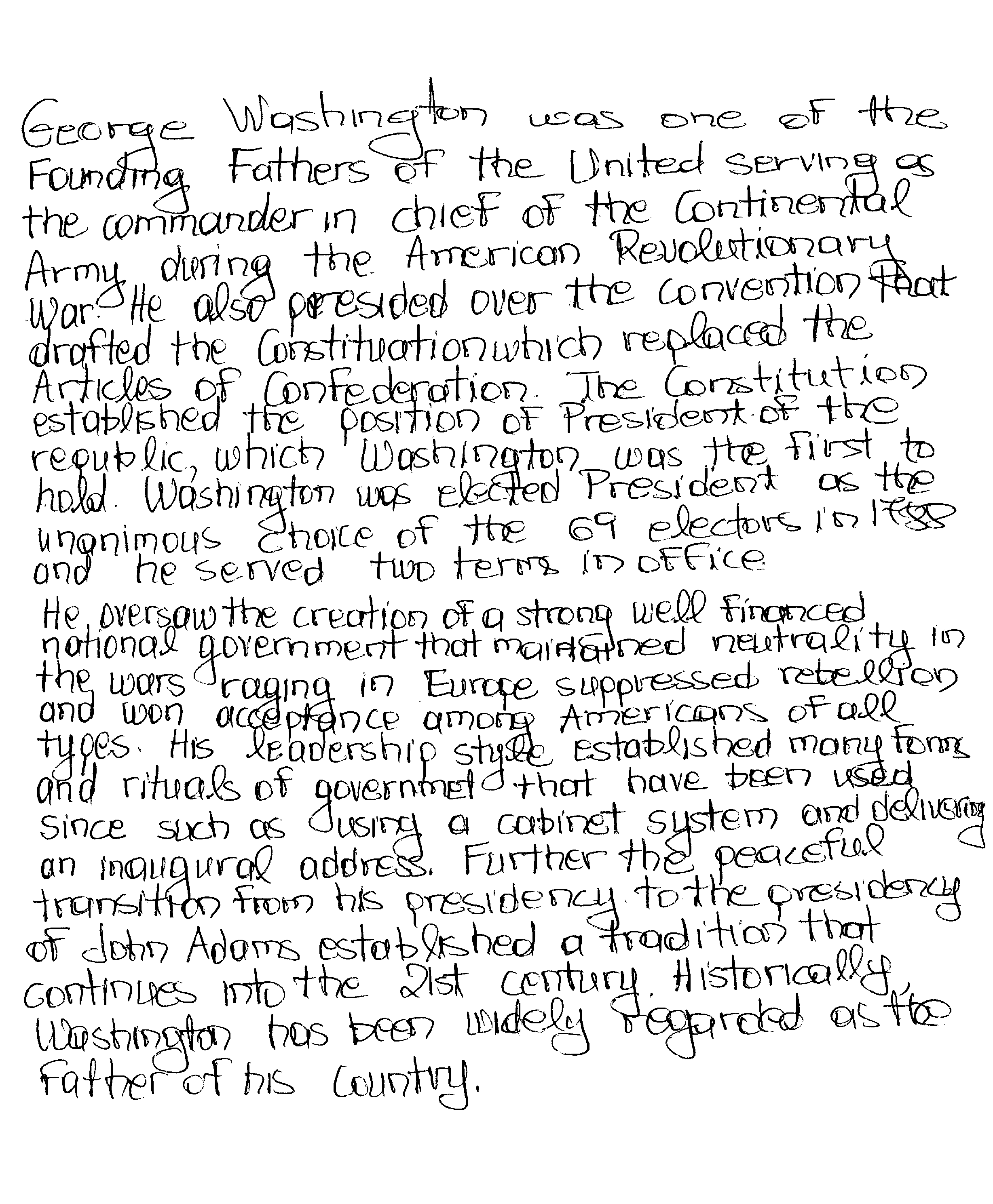}}
  \caption{}
\end{subfigure}
\begin{subfigure}{.33\textwidth}
  \centering
  \fbox{\includegraphics[width=.9\linewidth]{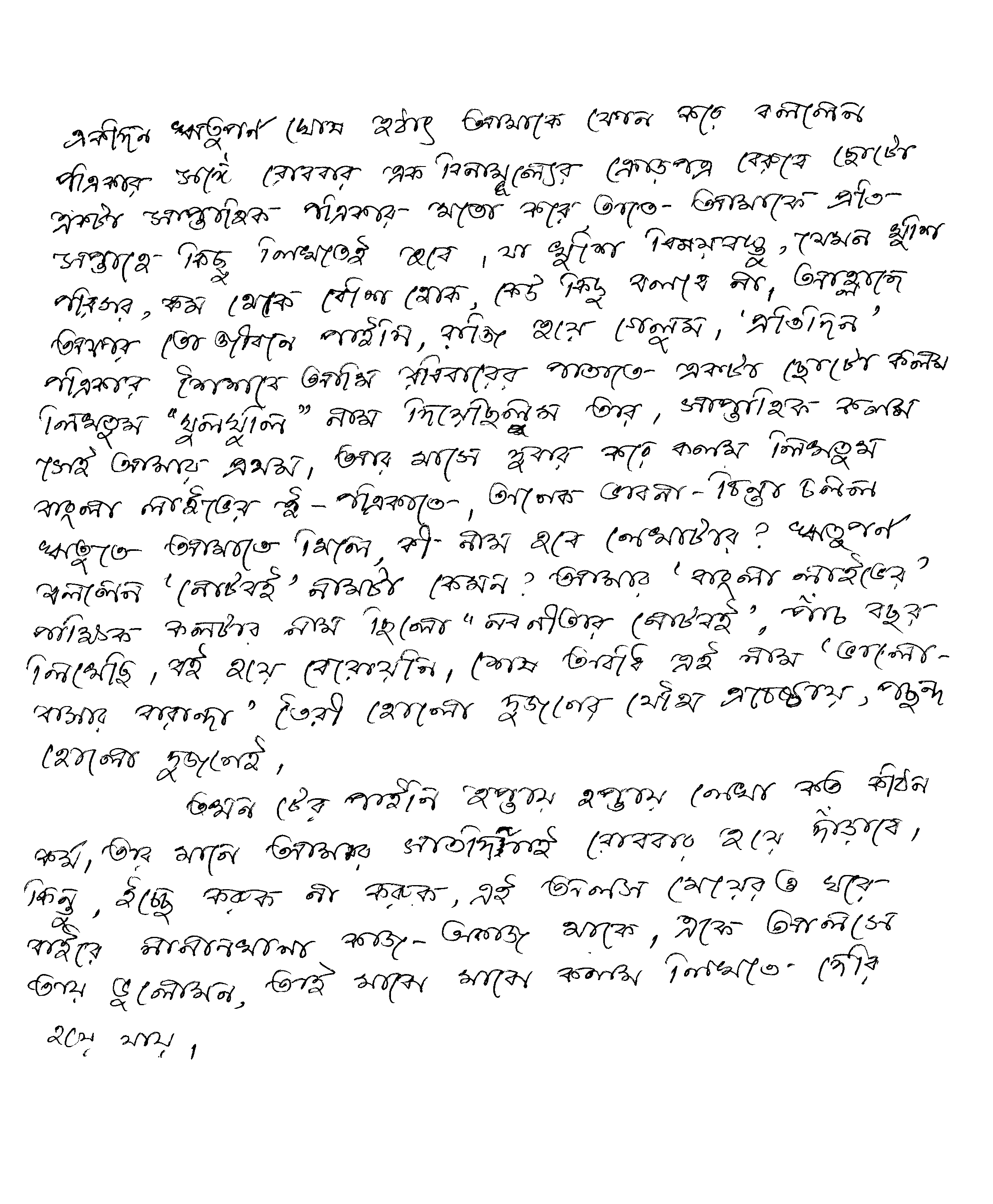}}
  \caption{}
\end{subfigure}
\caption{Three examples of document images from the ICDAR 2013 segmentation contest. This dataset contains documents written in several languages and free of graphical or non-text elements.}
\label{fig:ICDAR}
\end{figure}

On the other hand, we evaluate on the documents of the George Washington database~\cite{Fischer2012}. This database is composed of 20 gray-scale images from the George Washington Papers at the Library of Congress dated from the 18th century. The documents are written in English language in a longhand script. This database adds a set of different challenges with respect to the previous one due to the old script style, overlapping lines and a more complex layout. Also, documents may contain non-text elements as stamps or line separators. We show several examples in \figurename~\ref{fig:GW}.
We use the same ground truth introduced for this task in \cite{Fernandez2014} since there is not public ground truth for the task of line segmentation. For this reason, it is not possible to compare with any other methods apart from previous works and \cite{Fernandez2014}. We present the results as an indicator of the adaptability of our method to historical documents.

\begin{figure}[t]
\centering
\begin{subfigure}{.33\textwidth}
  \centering
  \includegraphics[width=.9\linewidth]{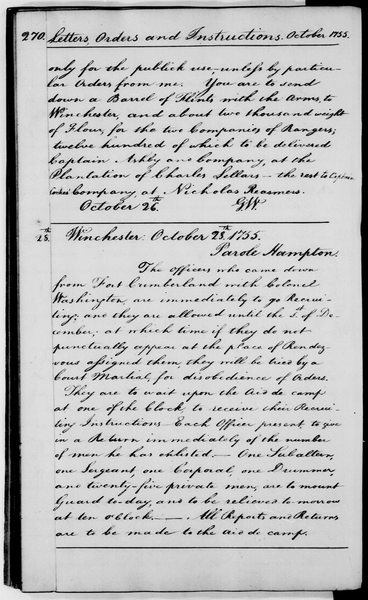}
  \caption{}
\end{subfigure}%
\begin{subfigure}{.33\textwidth}
  \centering
  \includegraphics[width=.9\linewidth]{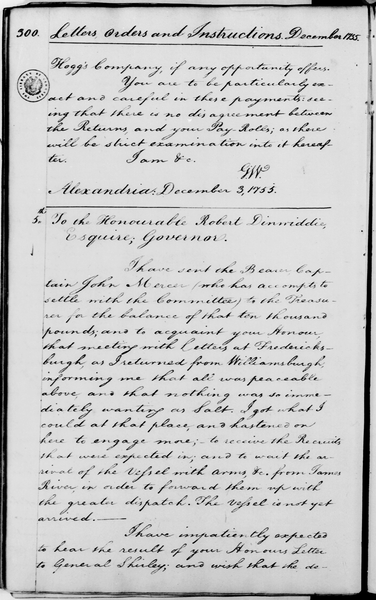}
  \caption{}
\end{subfigure}
\begin{subfigure}{.33\textwidth}
  \centering
  \includegraphics[width=.9\linewidth]{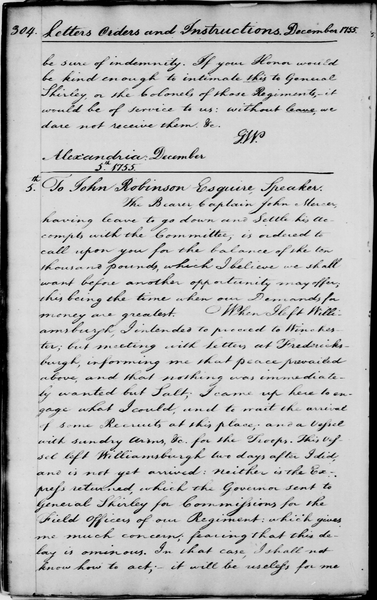}
  \caption{}
\end{subfigure}
\caption{Three examples of document images from the George Washington dataset. The dataset contains gray-scale document images including some non-text elements such as stamps or separator lines.}
\label{fig:GW}
\end{figure}

Last, we test our method in a collection of administrative documents with handwritten annotations. This is a more heterogeneous and complex dataset, since it contains documents with multiple text regions, each of them with different characteristics as orientation and writing style.
The collection includes letter-type documents, annotations in machine-printed documents, information from bank checks and other documents with complex layouts. 
The set of documents in the dataset is the result of the application of a previous machine-printed text separation \cite{Belaid2014}, in order to remove all possible not handwritten components. We apply the line segmentation algorithm on the handwritten layer without any particular filtering process 
The dataset is written in English and French languages and is composed of 433 document images. We show some examples of documents in \figurename~\ref{fig:DOD}.

\begin{figure} [t]
\centering
\begin{subfigure}{.33\textwidth}
  \centering
  \fbox{\includegraphics[width=.8\linewidth]{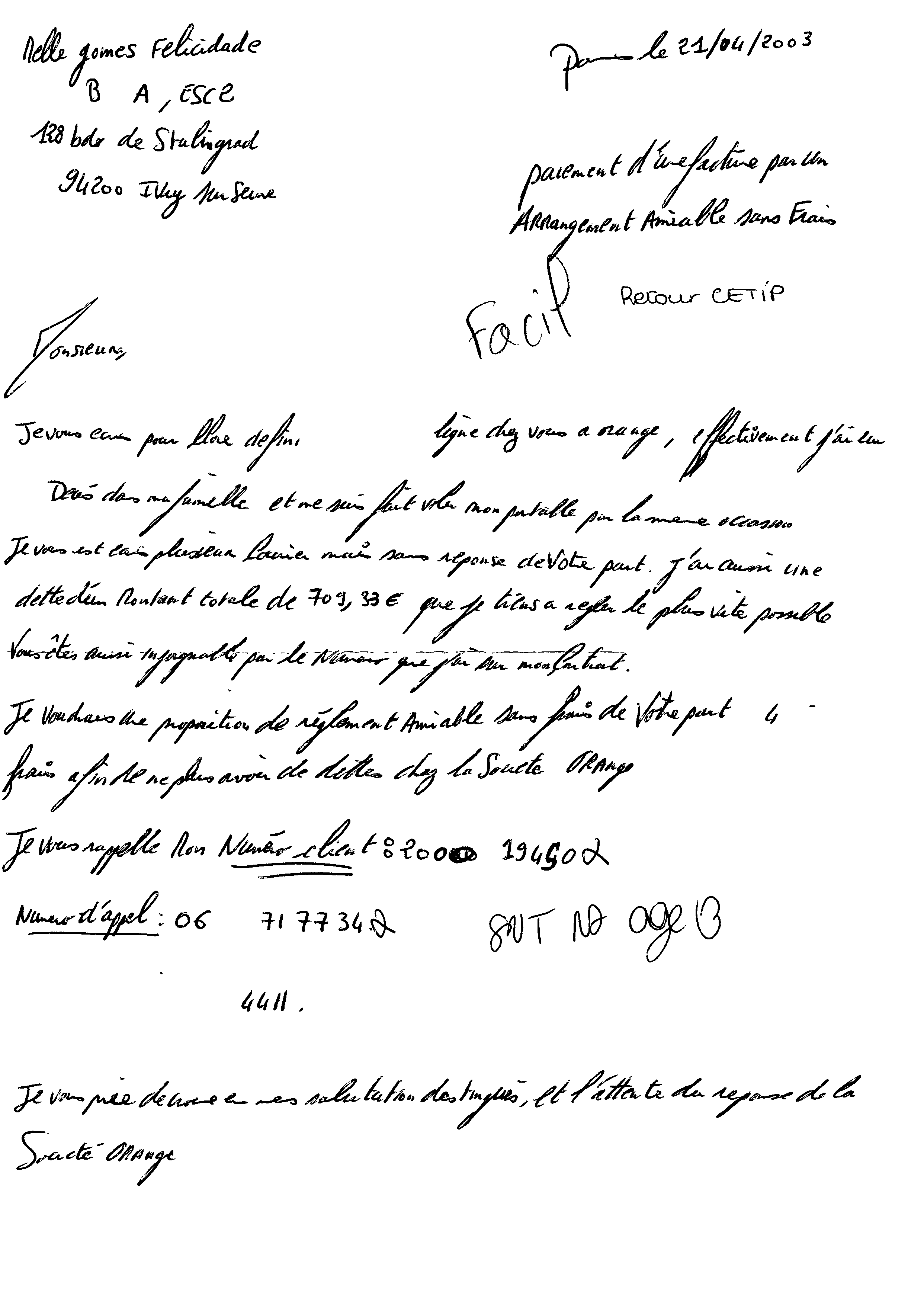}}
  \caption{}
  \label{fig:DODa}
\end{subfigure}%
\begin{subfigure}{.33\textwidth}
  \centering
  \fbox{\includegraphics[width=.8\linewidth]{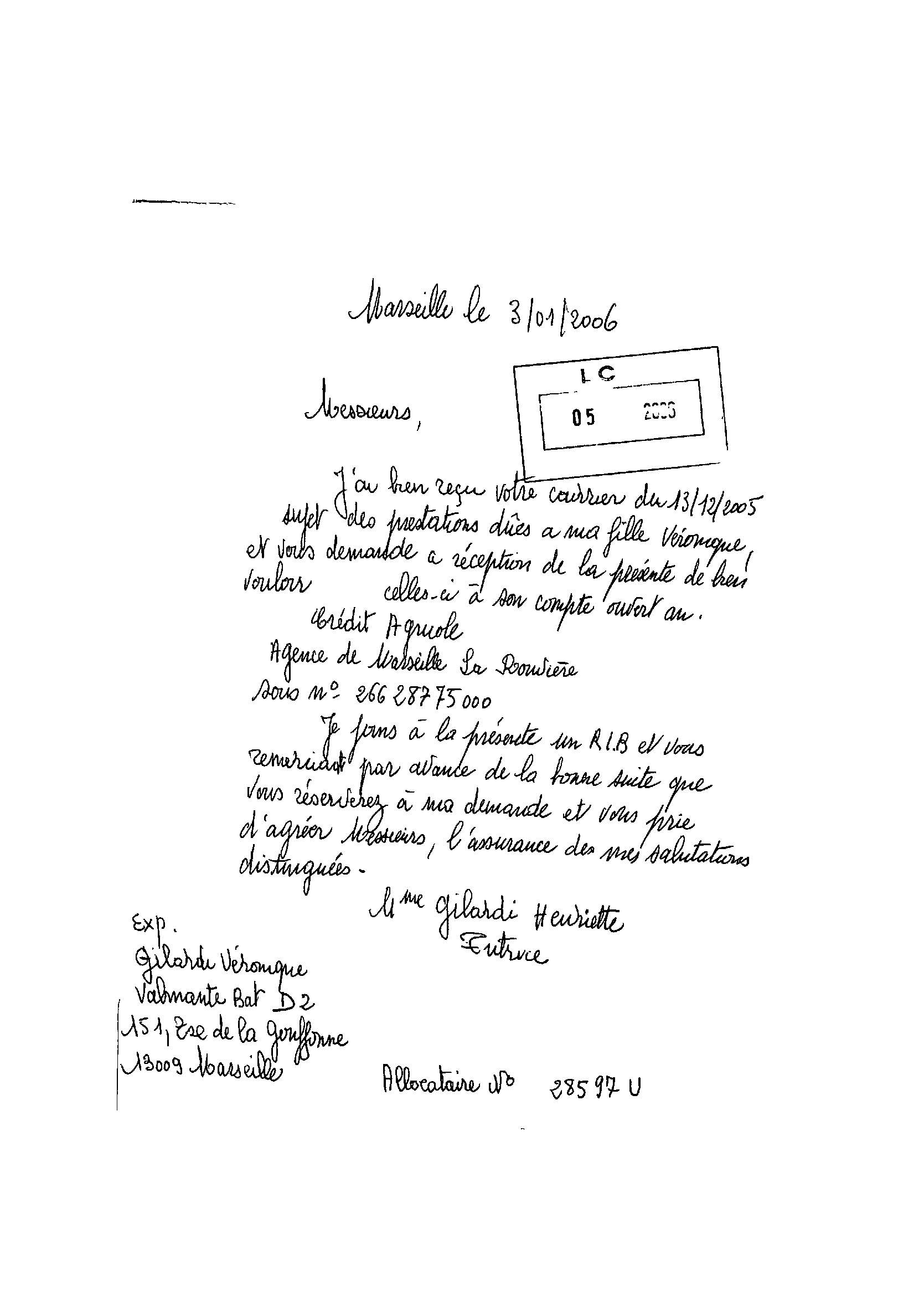}}
  \caption{}
  \label{fig:DODb}
\end{subfigure}
\begin{subfigure}{.33\textwidth}
  \centering
  \fbox{\includegraphics[width=.8\linewidth]{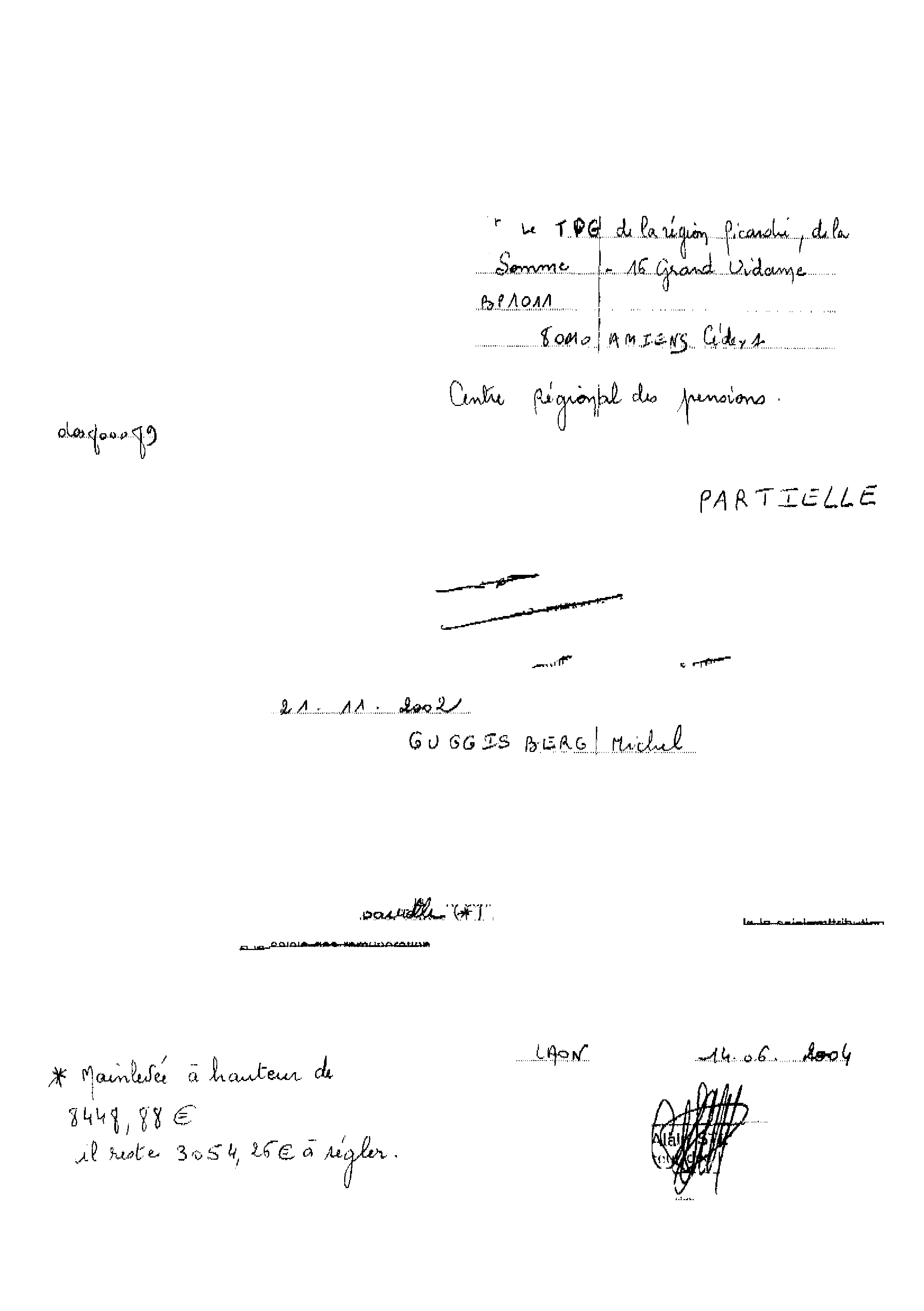}}
  \caption{}
  \label{fig:DODc}
\end{subfigure}
\caption{Three documents from the dataset of administrative annotated documents. In b) we observe some residues from the machine-printed separation step that could be interpreted as text components. The dataset contains documents with multiple layouts and text configurations.}
\label{fig:DOD}
\end{figure}

\subsection{Random pixel selection}

We aim at analyzing the impact of the density of random text pixels selected for the construction of the graphical model. We conduct this experiment on the ICDAR 2013 dataset for a pixel ratio of $1\%,3\%,5\%,10\%$ and $15\%$ of the total amount of text pixels. Table~\ref{tab:randompoints} shows the obtained results in terms of the F-measure, mean processing time, and its corresponding confidence intervals. 

We see that using values above $5\%$ do not produce significant improvements in the results, while the computational complexity increases considerably due to the large number of variables and connections in the MRF model.
With a $1\%$ of pixels, we obtain a $95.52\%$ in almost four times less computational time compared to $5\%$. However, the reduction in the number of pixels may leave some text regions uncovered, which can lead to an incorrect segmentation. 
Besides, we observe that the confidence interval for higher number of pixels increases. This implies that the method becomes less stable.
For the rest of experiments we select a $5\%$ of text pixels as standard value, since it seems to provide a good trade-off between data representation and time complexity.

\begin{table}[h]
\begin{center}
\begin{tabular}{c|c|c}
\hline
(\%) of points &  FM(\%) & Time(mean) \\
\hline \hline
1\%    & 95.52 $\pm$ 1.46  & 12.6s $\pm$ 1.1 \\
3\%    & 96.95 $\pm$ 1.24  & 28.4s $\pm$ 2.7 \\
5\%    & 97.05 $\pm$ 1.17  & 42.4s $\pm$ 3.3 \\
10\%   & 97.05 $\pm$ 1.18  & 81.3s $\pm$ 7.7 \\
15\%   & 97.05 $\pm$ 1.25  & 115.1s $\pm$ 11.2 \\
\hline
\end{tabular}
\end{center}
\caption{Results for different percentage of random points selected for the construction of the graphical model. We see as values above $5\%$ do not produce significant improvements in the results, while the computational complexity  considerably increases.}
\label{tab:randompoints}
\end{table}

\subsection{ICDAR segmentation contests}

We show in Table~\ref{tab:ICDAR2009} the results obtained on the ICDAR 2009 Handwriting Segmentation dataset. 
We obtain a $98.68\%$ FM value, with a confidence interval $[98.23, 99.13]$. This result compares with the top methods of the competition and overcomes the result obtained by previous works using a simpler probabilistic model \cite{Cruz2013}. 

In addition, the analysis of the results shows that $166/200$ images reach $100\%$ of FM, while the main errors are concentrated in a few error cases.
First type of error is related with the extra lines not removed in the post-process step that end up fitting diacritics or small isolated text components.  
This type of errors has a large impact in the numerical results, since implies an extra detection and may affect to several one-to-one text line associations. However, this error has no impact in posterior text recognition tasks, since text lines are usually well segmented. A severe case of extra line is shown in \figurename~\ref{fig:extralines}.
The second type of error is produced in areas where several touching characters converge. In this case, it is possible that the high connectivity within the MRF in this area favors the same labeling for all the text component instead of separating between text lines. An example of this last error can be seen in \figurename~\ref{fig:touchingwords}.

\begin{table}[t]
\begin{center} 
\begin{tabular}{  c | c | c | c | c | c } 
\hline
Method & M & o2o & {DR (\%)} & {RA (\%)} & {FM (\%)} \\ \hline \hline
CUBS  		  & 4036 & 4016 & 99.55 & 99.50 & 99.53 \\  
ILSP-LWSeg-09 & 4043 & 4000 & 99.16 & 98.94 & 99.05 \\  
HandwritingPAIS   & 4031 & 3973 & 98.49 & 98.56 & 98.52 \\  
CMM 			  & 4044 & 3975 & 98.54 & 98.29 & 98.42 \\  
Fernandez \textit{et al.} \cite{Fernandez2014}  & 4176 & 3971 & 98,40 & 95,00 & 96,67 \\ 
CASIA-MSTSeg  & 4049 & 3867 & 95.86 & 95.51 & 95.68 \\  
Cruz \textit{et al}. \cite{Cruz2013} & 4061 & 3858 &	95.60 & 95.00 & 95.20 \\ 
PortoUniv     & 4028 & 3811 & 94.47 & 94.61 & 94.54 \\ 
PPSL  		  & 4084 & 3792 & 94.00 & 92.85 & 93.42 \\ 
LRDE 	  	  & 4423 & 3901 & 96.70 & 88.20 & 92.25 \\ 
Jadavpur Univ & 4075 & 3541 & 87.78 & 86.90 & 87.34 \\ 
ETS 			  & 4033 & 3496 & 86.66 & 86.68 & 86.67 \\  
AegeanUniv    & 4054 & 3130 & 77.59 & 77.21 & 77.40 \\  
REGIM         & 4563 & 1629 & 40.38 & 35.70 & 37.20 \\ 
\hline
\hline
{Proposed} & {4044} & {3986} &	{98.81} & {98.56} & {98.68} \\ 
\hline
\end{tabular}
\end{center} 
\caption{Results on the ICDAR2009 handwriting segmentation contest~\cite{Gatos2009}. We improve our previously reported results and obtain FM values up to other state-of-the-art methods.} 
\label{tab:ICDAR2009} 
\end{table}

In Table~\ref{tab:ICDAR2013} we show the result obtained on the ICDAR 2013 Handwriting Segmentation dataset. The additional complexity of this dataset is reflected in the results, where we obtain a $97.05\%$ FM value with a confidence interval $[95.80, 98.30]$.
In comparison with the rest of the methods we see that our method is slightly below the top methods in quantitative terms, although it overcomes many of them. However, we report a total of $125/150$ images labeled with a $100\%$ FM, and according to the confidence interval we can say that our method is stable along all the dataset.

As in the previous experiment, the $80\%$ of the errors are related to extra lines fitting isolated components. The remaining $20\%$ is related to the new characteristics of this dataset. For instance, some crowded images where our overlapping detector is not able to split them, \figurename~\ref{fig:severalerrors}. Nevertheless, in the practice our method is able to deal with the majority of these situations as seen in \figurename~\ref{fig:challenges}.

\begin{table}[t]
\begin{center}
\begin{tabular}{c|c|c|c|c|c}
\hline
Method     &   M  &  o2o &  DR(\%) &  RA(\%) &  FM(\%) \\
\hline \hline
INMC  	   & 2614 & 2614 &  98.68  &  98.64  &  98.66  \\
NUS  	   & 2645 & 2605 &  98.34  &  98.49  &  98.41  \\
GOLESTAN-a & 2646 & 2602 &  98.23  &  98.34  &  98.28  \\
CUBS  	   & 2677 & 2595 &  97.96  &  96.94  &  97.45  \\
IRISA      & 2674 & 2592 &  97.85  &  96.93  &  97.39  \\
LRDE  	   & 2632 & 2568 &  96.94  &  97.57  &  97.25  \\
Fernandez \textit{et al.} \cite{Fernandez2014}  & 2697 & 2551 & 96,30 & 94,58 & 95,43 \\ 
QATAR-b	   & 2609 & 2430 &  91.73  &  73.14  &  92.43  \\
MSHK   	   & 2696 & 2428 &  91.66  &  90.06  &  90.85  \\
CVC   	   & 2715 & 2418 & 91.28 & 89.06 & 90.16 \\
\hline
\hline
{Proposed}   & {2647} & {2570} &  {97.01}  & {97.09}  &  {97.05} \\
\hline
\end{tabular}
\end{center}
\caption{Results on the ICDAR2013 handwriting segmentation contest~\cite{Gatos2013}. Our method still obtains results up to the other contestant methods and improves our previously reported results.} 
\label{tab:ICDAR2013} 
\end{table}

\begin{figure}[t]
	\centering
	\includegraphics[width=.95\linewidth]{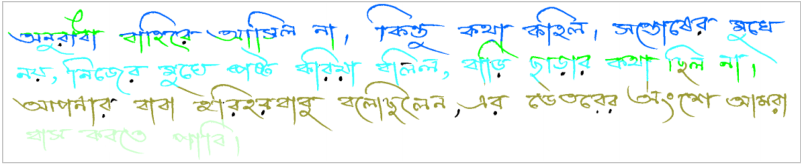}
	\caption{Example of the impact of an extra regression line not removed in the post-process step. In this severe case several lines are affected by the incorrect labeling of some of their text pixels to the extra regression line (text line in green).}
\label{fig:extralines}
\end{figure}

\begin{figure}[t]
	\centering
	\includegraphics[width=.95\linewidth]{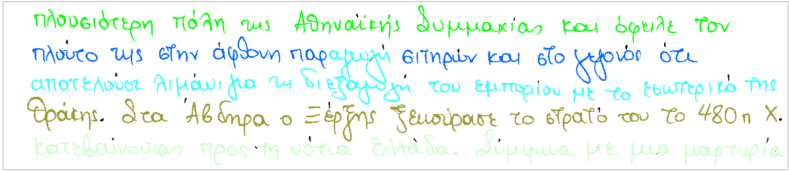}
	\caption{Segmentation error in a region with several touching characters. This is an example of how the messages sent between variables from the involved words favor the same labeling for every component in conflict.}
\label{fig:touchingwords}
\end{figure}

\begin{figure}[t]
\begin{center}
	\includegraphics[width=.95\linewidth]{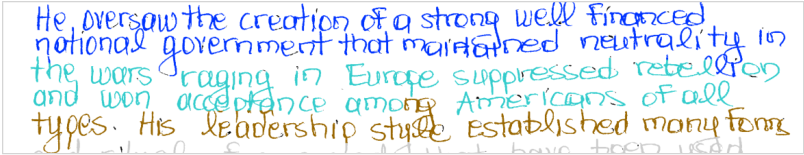}
	\caption{Example of several overlapped lines that were not detected in the initialization step. This produces an incorrect number of initial candidates which is difficult to overcome in the rest of the process.}
\label{fig:severalerrors}
\end{center}
\end{figure}

\subsection{George Washington dataset}

Table~\ref{tab:GW} shows the results obtained on the George Washington dataset. We improve our previous results with a $90.06\%$ FM with an increase of the DR of almost a $10\%$. Low RA values are caused by non-text components, since we have not considered some features of this dataset that in other circumstances could be integrated for a better performance. However, we want to use the same model configuration for all the experiments in order to evaluate the adaptability to other datasets without parameter tuning or additional training. 

For instance, the detection of non-text elements as stamps, text line separators, or underlines affects to the numerical results. In the case of underlines, they are labeled as text line component, while our method recognizes it separately. The same effect happens in some arbitrary separator lines. 
As in the case of the extra text line detection, this error has an impact in the numerical results, although the final set of detected text lines is usually correct. Again, our method obtains better results in views of posterior text recognition tasks, since it is able to separate text from other non-textual components.
Results on this dataset prove the capability of our method for segmenting text lines in historical documents without the need of reconfiguration.

\begin{table} [t]
\begin{center} 
\begin{tabular}{  c | c | c | c | c | c } 
\hline 
Method & M & o2o & {DR (\%)} & {RA (\%)} & {FM (\%)} \\ 
\hline \hline
Fernandez \textit{et al.}\cite{Fernandez2014} 	& 693 & 653 & 91,30 & 94,20 & 92,70 \\ 
Cruz \textit{et al.} \cite{Cruz2013} & 631 & 551 & 82,60 & 87,30 & 84,80 \\ 
Base line \cite{Fernandez2014}  & 727 & 338 & 47,20 & 46,40 & 46,70 \\ 
\hline
\hline
Proposed & 702 & 614 & 92.05 & 88.16 & 90.06 \\ 
\hline
\end{tabular}
\end{center} \caption{Results on the George Washington dataset. Lower o2o and RA rates are produced due to the identification of separator lines as independent text line. However, qualitative results show that our method produces a better detection of text-lines with respect to the other compared methods.} 
\label{tab:GW}
\end{table}

\subsection{Administrative annotated documents}

As for the GW dataset, we use the same parameter configuration than for the ICDAR experiments. 
In this dataset one of the challenges is to detect the different text regions in order to process them separately. For instance, in \figurename~\ref{fig:DODa} we can appreciate at the bottom of the central block of text three text lines with different font that have to be labeled separately. We can see another example in \figurename~\ref{fig:DODb}, where several lines at the bottom of the document may be merged as the same line in the case of processing the full page.
Our method behaves on these cases on two possible ways. In most of the cases the different text regions are detected in the initialization step and then processed separately. 
However, in some documents where the region segmentation is not achieved, text lines are approximated by several regression lines due to the initial over-segmentation.

On this experiment we are not able to compare to other works, since it is a non-published collection of documents, however we can compare against our previous work in order to validate the new model.
Table~\ref{tab:DOD} shows the results obtained. We see that the results are significantly improved with the new proposed approach. The observed improvement confirm the contribution of the proposed model.
In addition, the result on this dataset proves the versatility of our method on complex layouts.

\begin{table}[t]
\begin{center}
\begin{tabular}{c|c|c|c}
\hline
           &  Precision   &  Recall  &  FM(\%) \\
\hline \hline
Cruz \textit{et al}. \cite{Cruz2013}  & 69.45 & 72.38  &  70.88  \\
{Proposed}  	 & {79.75} & {82.70}  &  {81.19}  \\
\hline
\end{tabular}
\end{center}
\caption{Results on the administrative handwritten annotation dataset. We significantly improve the results obtained by our previous model.}
\label{tab:DOD}
\end{table}

\begin{figure}
\begin{center}
	\includegraphics[width=.80\linewidth]{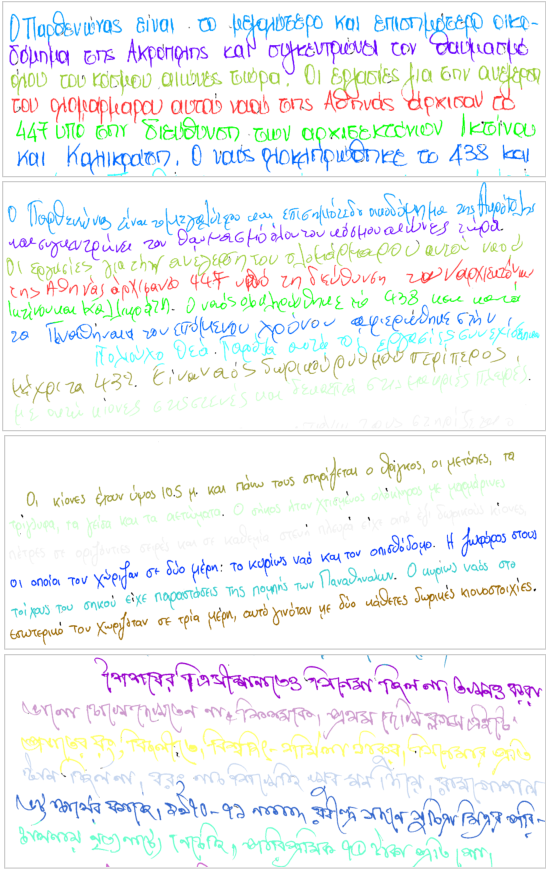}
	\caption{Results of four fragments of images from the ICDAR 2013 dataset that include crowded text and light curvatures. We can see several challenging scenarios where our method performs correctly.}
	\label{fig:challenges}
\end{center}
\end{figure}

\newpage

\section{Conclusion}
\label{Concl}

In this paper we present a general method for handwriting text line segmentation based on the estimation of a set of regression lines. 
We propose a probabilistic framework that relies on the EM algorithm, for the estimation of the regression parameters, and on a MRF model, for parameter learning of neighboring pixels. We implement a message-passing-based algorithm to compute approximate inference and learning the model parameters.
Our method can be applied on documents with different layouts and features. Besides, our framework permits to easily extend the model with the inclusion of prior information by means of new feature functions.

We conduct several experiments with promising results on four collections of documents without model reconfiguration. The selected datasets include several types of layouts, historical an contemporary documents, from several writers and scripts. Besides, our method is able to deal with most of the situations regarding touching text lines and light curvatures of the text, which demonstrates the contributions of the proposed model.
The results validate our initial hypothesis, since we prove that a set of regression lines can fit with high accuracy the actual text lines locations.

As future work lines, we consider to use higher order regression model that could lead to better approximation of curved and complex lines. Besides, we think that some of the current errors may be corrected with the inclusion of more informative feature functions. Note that we use a reduced and basic set of feature functions based on the pixel location and common pairwise interactions. We believe that the inclusion of more specific knowledge could improve the overall results. For instance, we can incorporate improved pairwise feature that analyzes the edge length and relative position between the connected variables. 
In addition, we plan to integrate new discriminant features in order to perform machine-printed/handwritten text separation and line segmentation within the same process.
In this way we will be able to process administrative annotated documents directly without the need of previous steps.

\section*{Acknowledgements}

This work has been partially supported by the Spanish project TIN2015-70924-C2-2-R.

\section*{Bibliography}

\bibliographystyle{elsarticle-num} 
\bibliography{thesisbib}

\setcounter{section}{0}

\title{Supplementary Material: A probabilistic framework for handwritten text line segmentation}

\author{Francisco~Cruz, Oriol~Ramos~Terrades}

\maketitle

This document contains the supplementary material for the paper: \textit{A probabilistic framework for handwritten text line segmentation}. We  include the derivation of the update equations for linear regression parameters. Besides, we show the derivation steps of Algorithm \ref{alg:MP} for inference and parameter learning. 

\newpage

\section{Update equations for linear regression parameters }

In this  section we provide  a  complete derivation of linear regression parameters for the the  model introduced in section 3.

\begin{prop}[Joint partition function]
Given the factorization of Eq. (1)  of paper:
\begin{equation}
p(e,h|\Theta) = \prod_vp_v(e_v|h_v,\Theta_a)p(h|\Theta_b)
\end{equation}
The  partition function  $Z(\Theta)$ of the joint distribution  $p(e,h|\Theta)$ is the partition function of the {\em a priori} distribution of hidden variables $h$, $Z_0(\Theta_b)$.
\end{prop}

\begin{proof}
Straightforward from  the  definitions of partition functions:
\begin{equation}
Z(\Theta) =  \sum_h  \int \prod_vp_v(e_v|h_v,\Theta_a)p(h|\Theta_b) de
\end{equation}

\noindent where the  sum runs over the  values  of all  hidden  variables and  the integral domain is $\mathbb{R}^{2N}$, which correspond to  the coordinates of the $N$ observed values $e_v$. A simple  reordering  of the  integral operations leads to  the final result, since  $p_v(e_v|h_v,\Theta_a)$ is a  pdf  and  the  integral is 1:

\begin{equation}
Z(\Theta) \triangleq   \sum_h  \prod_v \int p_v(e_v|h_v,\Theta_a)de_v p(h|\Theta_b)  = \sum_h   p(h|\Theta_b) \triangleq  Z_0(\Theta_b)
\end{equation}

\end{proof}

We defined the conditional likelihood probability $p_v(e_v|h_v,\Theta_a)=\frac{\Psi_v(h_{v},e_{v}| \Theta_a)}{Z_v(h_v,\Theta_a)}$, where $Z_v(h_v,\Theta_a)$ is its corresponding partition function defined as:

\begin{equation}
Z_v(h_v,\Theta_a) =  \int \exp\left\{ \sum_{k \in I_{v}}  g_k(h_{v},e_{v}| \Theta_a) \right\}de_v
\end{equation}

Given the above definitions the  conditional expectation $Q$ used to  derive the EM algorithm becomes:

\begin{equation}
\label{eq:sup_Q}
\begin{split}
Q(\Theta | \Theta') &= \sum_u \sum_{k \in I_u} \left[ \sum_{h_u} f_k(h_u | \Theta_b) p_u(h_u|\Theta_b') \right] - \log Z_0(\Theta_b)   + \\
&+  \sum_v  \sum_{h_v}\log p_v(e_v|h_v,\Theta_a) p_v(h_v|e_v,\Theta_a',\Theta_b') \\
&= \sum_v   \sum_{h_v}  \left[\sum_{k \in I_v}g_k(h_v,e_v | \Theta_a)  -  \log Z_v(h_v,\Theta_a)\right]p_v(h_v|e_v,\Theta_a',\Theta_b')  + \\
&+  \sum_u \sum_{k \in I_u} \left[ \sum_{h_u} f_k(h_u | \Theta_b) p_u(h_u|\Theta_b') \right] - \log Z_0(\Theta_b)
\end{split}
\end{equation}

We adapt the EM algorithm to update the parameters in $\Theta_a$, while we keep fix the parameters in $\Theta_b$. We will update later $\Theta_b$ by means of Algorithm~\ref{alg:MP}. The partial derivative of $Q$ for a parameter $\theta_k \in \Theta_a$ is:

\begin{equation}
\label{EqPartial}
\begin{split}
\frac{\partial }{\partial\theta_k}Q(\Theta | \Theta')  &= \sum_v   \sum_{k \in I_v} \sum_{h_v} \frac{\partial }{\partial\theta_k} g_k(h_v,e_v | \Theta_a) p_v(h_v|e_v,\Theta_a',\Theta_b) -  \\ 
&- \sum_v \sum_{h_v}\frac{\partial}{\partial\theta_k}\log Z_v(h_v,\Theta_a)p_v(h_v|e_v,\Theta_a',\Theta_b)= 0 
\end{split}
\end{equation}

In case we model text lines by the  likelihood probabilities  below:
\begin{equation}
\label{eq:Gauss1}
\begin{split}
p_t(x_v,y_v | h_v=l,\theta_a) &\propto \exp \left\{ -\frac{(y_v - a_l x_v - b_l)^2}{2 \sigma_{l,t}^2}  \right\} \\
p_s(x_v,y_v | h_v=l,\theta_a) &\propto \exp \left\{ -\frac{(x_v - c_l)^2}{2 \sigma_{l,s}^2}  \right\}
\end{split}
\end{equation}

The $v$-terms of  Eq.~\eqref{eq:sup_Q},  after  some manipulations and ignoring terms, which will disappear after taking partial derivatives, can be written as:

\begin{equation}\label{eq:Q_v}
\sum_v \sum_{h_v} \left( -\frac{1}{2}(A_le_v-\mu_l)^t\Sigma_l^{-1}(A_le_v-\mu_l) - \frac{1}{2}\log | \Sigma_l|\right)p(h_v=l|e_v,\Theta'_a,\Theta_b) 
\end{equation}

\noindent where $A_l=\left(\begin{array}{cc}1 & 0 \\ -a_l & 1\end{array}\right)$, $\Sigma_l=\left(\begin{array}{cc}\sigma^2_{l,t} & 0 \\0 & \sigma^2_{l,s} \end{array}\right) $ and $\mu_l=\left(\begin{array}{c}c_l \\ b_l\end{array}\right)$. Moreover, $ | \Sigma_l|$ is the partition function of the conditional likelihood probability given by  the next Proposition:

\begin{prop}[Partition function of conditional likelihood probability]
The partition function $Z_v(h_v=l,\Theta_a)$ is:
\begin{equation}
Z_v(h_v=l,\Theta_a) =  2\pi| \Sigma_l|^{1/2}
\end{equation}
\end{prop}
\begin{proof}
The  result is  straightforward  after  basic calculus and taking into account that the  partition function of a multivariate normal distribution  with covariance matrix $\Sigma$ is:  $(2\pi)^{k/2}|\Sigma|^{1/2}$, where $k$ herein is the dimension of such multivariate normal distribution.

We recall that, in Eq.~\eqref{eq:Gauss1}, we defined $g_k$ as:
\begin{equation}\label{eq:gk}
g_k(h_v,e_v|\Theta_a) = -\frac{(y_v-a_lx_v-b_l)^2}{2\sigma^2_{l,t}}-\frac{(x_v-c_l)^2}{2\sigma^2_{l,s}}
\end{equation}

Which can be written, for  $h_v=l$, in matrix form:

\begin{equation}
g_k(l,e_v|\Theta_a) = -\frac{1}{2}(A_le_v-\mu_l)^t\Sigma_l^{-1}(A_le_v-\mu_l)
\end{equation}

\noindent where $A_l$, $\mu_l$ and $\Sigma_l$ are defined as above. The  partition function for $h_v=l$ is  computed  as:

\begin{equation}
Z_v(h_v=l,\Theta_a) = \int \exp\left\{  -\frac{1}{2}(A_le-\mu_l)^t\Sigma_l^{-1}(A_le-\mu_l) \right\}de
\end{equation}

A  simple change of coordinate, taking into account that $|A_l|=1$ for all $l$, and setting $S_l^{-1}=A_l^t\Sigma_l^{-1}A_l$ lead us to:

\begin{equation}
Z_v(h_v=l,\Theta_a) = \int \exp\left\{  -\frac{1}{2}(u-A_l^{-1}\mu_l)^tS_l^{-1}(u-A_l^{-1}\mu_l) \right\}du
\end{equation}

\noindent which is the  definition of a multivariate normal distribution of 2 dimensions. The result follows from the  properties of the  matrix determinant:
\begin{equation}
Z_v(h_v=l,\Theta_a) =  2\pi|S_l|^{1/2} = 2\pi|\Sigma_l|^{1/2}
\end{equation}
\end{proof}

The authors in~\cite{Bilmes1997} provide the derivation of the update parameter formulas for a gaussian mixture model. The update formulas introduced in this paper are essentially the same but adapted to regression lines and our model. Moreover, notice that Eq.~\eqref{eq:Q_v} are almost equal to Eq.~(7) in that work. It means that the  arguments introduced there also applies to our method.

More specifically, variances  $\sigma^2_{l,t}$ and $\sigma^2_{l,s}$ and  {\em mean} parameter $c_l$ comes straightforward from~\cite{Bilmes1997}. To illustrate it we show in what follows that $c_l$ is equal to $\mu^{new}_l$ in that  paper.  $c_l$ only appears in the second  term of $g_k$ in Eq.~\eqref{eq:gk}. The partial derivative with respect to $c_l$ is:

\begin{equation}
\frac{\partial}{\partial c_l}Q(\Theta|\Theta')=\sum_v \frac{(x_v-c_l)}{\sigma^2_{l,s}}p(h_v=l|e_v,\Theta_a',\Theta_b) = 0
\end{equation}

Which, after rearraging the terms, we  can isolate $c_l$ and find the update formula:
\begin{equation}
c_l = \frac{\sum_v x_vp(h_v=l|e_v,\Theta_a',\Theta_b)}{\sum_v p(h_v=l|e_v,\Theta_a',\Theta_b)}
\end{equation}

To find the update formulas for  $\sigma^2_{l,t}$ and $\sigma^2_{l,s}$, we have to apply  the results  from  matrix algreball recalled in~\cite{Bilmes1997} to  covariance matrix $S_l^{-1}=A_l^t\Sigma_l^{-1}A_l$ and {\em mean} vector $A_l^{-1}\mu_l$. Thus,  matrix $M_{l,1}$ there becomes $M_{l,v}$ and it is defined  as:
\begin{equation}
M_{l,v} =  \Sigma_l - (A_le_v-\mu_l)(A_le_v-\mu_l)^t
\end{equation}

which  leads to the update formula:
\begin{equation}\label{eq:Sigma}
\Sigma_l = \frac{\sum_v(A_le_v-\mu_l)(A_le_v-\mu_l)^tp(h_v=l|e_v,\Theta_a',\Theta_b)}{\sum_vp(h_v=l|e_v,\Theta_a',\Theta_b)}
\end{equation}


To conclude, it remains to find the update formulas for the regression line parameters $a_l$ and  $b_l$. One can find the derivation of such formulas in any  textbook. We follow the same ideas. We start by computing the independent term $b_l$ and then we will find the slope $a_l$. The partial derivative of $Q$ with respect to $b_l$ is:

\begin{equation}\label{eq:bl}
\frac{\partial}{\partial b_l}Q(\Theta|\Theta')=\sum_v \frac{(y_v-a_lx_v-b_l)}{\sigma^2_{l,t}}p(h_v=l|e_v,\Theta_a',\Theta_b) = 0
\end{equation}
 
We do exactly  the same that we did for $c_l$ and we find:
\begin{equation}
b_l = \frac{\sum_v (y_v-a_lx_v)p(h_v=l|e_v,\Theta_a',\Theta_b)}{\sum_v p(h_v=l|e_v,\Theta_a',\Theta_b)}
\end{equation}

\noindent Which can easily be computed if we know $a_l$. The partial derivative with respect to $a_l$ is:
\begin{equation}
\frac{\partial}{\partial a_l}Q(\Theta|\Theta')=\sum_v \frac{x_v(y_v-a_lx_v-b_l)}{\sigma^2_{l,t}}p(l|e_v,\Theta_a',\Theta_b) = 0
\end{equation}
We replace in the above expression  $b_l$ by its definition in Eq.~\eqref{eq:bl} to obtain:
\begin{align}\label{eq:a1}
\sum_v x_v(y_v-a_lx_v)p(l|e_v,\Theta_a',\Theta_b) - \sum_v (y_v-a_lx_v)p(l|e_v,\Theta_a',\Theta_b)\bar{x}=&0
\end{align}

\noindent where we define the {\em mean} of the horizontal coordinates $\bar{x}$ as:
\begin{equation}
\bar{x} = \frac{\sum_v x_vp(l|e_v,\Theta_a',\Theta_b)}{\sum_v p(l|e_v,\Theta_a',\Theta_b)}
\end{equation}

We rearrange Eq.~\eqref{eq:a1} and we find :
\begin{equation}
\sum_v (x_v-\bar{x})(y_v-a_lx_v)p(h_v=l|e_v,\Theta_a',\Theta_b) 
\end{equation}

The remainder is  straightforward  but taking into account that $\bar{y}$ is defined as $\bar{x}$:
\begin{equation}
\bar{y} = \frac{\sum_v y_vp(h_v=l|e_v,\Theta_a',\Theta_b)}{\sum_v p(h_v=l|e_v,\Theta_a',\Theta_b)}
\end{equation}

Therefore, $a_l$:

\begin{equation}
a_l = \frac{\sum_v (x_v-\bar{x})(y_v-\bar{y})p(h_v=l|e_v,\Theta_a',\Theta_b)}{\sum_v (x_v-\bar{x})^2p(h_v=l|e_v,\Theta_a',\Theta_b) }
\end{equation}

\subsection{Rotation invariant updates}

The proposed regression model proposed in Eq.~\eqref{eq:Gauss1} is not fully rotation invariant, since we apply a shear transform $A_l$ instead of a rotation transform. In this section we introduce a rotation invariant model along its corresponding update formula. As we will see it only changes the update formula for the slope parameter $a_l$ and independent term $c_l$ while the update formulas for the other parameters remains equal.

\newpage

We define the new feature function $g_k$ as follows:

\begin{equation}\label{eq:gk_invariant}
g_k(l,e_v|\Theta_a) = -\frac{(y_v-a_lx_v-b_l)^2}{2\sigma^2_{l,t}}-\frac{(x_v+a_ly_v-c_l)^2}{2\sigma^2_{l,s}}
\end{equation}

Which can be written, for  $h_v=l$, in matrix form:

\begin{equation}
g_k(l,e_v|\Theta_a) = -\frac{1}{2}(R_le_v-\mu_l)^t\Sigma_l^{-1}(R_le_v-\mu_l)
\end{equation}

\noindent where $R_l=\left(\begin{array}{cc}1 & a_l \\ -a_l & 1\end{array}\right)$; and  $\mu_l$ and $\Sigma_l$ are defined as above. Recall that $a_l$ is the slope of the regression line, $a_l=\frac{\sin \beta}{\cos \beta}$. It becomes clear that $R_l$ is a rotation matrix of $-\beta$ radians and the model given by Eq.~\eqref{eq:gk_invariant} is rotation invariant.

The update formula is obtained similarly than before but it appears a new term. Thus, the partial derivative with respect to $a_l$ is:

\begin{equation}
\frac{\partial}{\partial a_l}Q=\sum_v \left(\frac{x_v(y_v-a_lx_v-b_l)}{\sigma^2_{l,t}}- \frac{y_v(x_v+a_ly_v-c_l)}{\sigma^2_{l,s}}\right)p(l|e_v,\Theta_a',\Theta_b) = 0
\end{equation}

After some calculations the final update formula is:

\begin{equation}
a_l = 2\frac{\sum_v (x_v-\bar{x})(y_v-\bar{y})p(h_v=l|e_v,\Theta_a',\Theta_b)}{\sum_v \left((x_v-\bar{x})^2 + (y_v-\bar{y})^2\right)p(h_v=l|e_v,\Theta_a',\Theta_b)}
\end{equation}

\noindent and $c_l$ is:
\begin{equation}
c_l = \frac{\sum_v (x_v+a_ly_v)p(h_v=l|e_v,\Theta_a',\Theta_b)}{\sum_v p(h_v=l|e_v,\Theta_a',\Theta_b)}
\end{equation}

Covariance matrix $\Sigma_l$ is computed like in Eq.~\eqref{eq:Sigma} but replacing $A_l$ by $R_l$.

\newpage

\section{Derivation of Algorithm~\ref{alg:MP} }

The optimization problem formulated in Eq.~(9) results on a constrained minimization problem that can be solved by means of Lagrange multipliers as:

\begin{equation}\label{lagrangian}
\begin{split}
L(p_u, p_v, \Lambda, \Theta, N) = \\
\sum_u \sum_{h_u} p_u(h_u|\Theta_b) \log p_u(h_u|\Theta_b) + \\
+ \sum_v c_v\sum_{h_v} p_v(h_v|e_v,\Theta_a',\Theta_b) \log p_v(h_v|e_v,\Theta_a',\Theta_b) + \\ 
+ \sum_k \sum_{h_v} g_k(h_v,e_v|\Theta_a')p_v(h_v|e_v,\Theta_a',\Theta_b) + \\
+ \sum_k \theta_k \Bigg[ \sum_u f_k(h_u) p_u(h_u|\Theta_b) - \mu_{k}  \Bigg] + \\
+ \sum_v \sum_{u \supset v} \sum_{h_v} \lambda_{v \rightarrow u} (h_v) \Bigg[ \sum_{h_{u \setminus v }} p_u(h_u|\Theta_b) -p_v(h_v|e_v,\Theta_a',\Theta_b) \Bigg] +  \\
+  \sum_u \nu_u \Bigg[ \sum_{h_u} p_u(h_u|\Theta_b) - 1  \Bigg] +\sum_v \nu_v \Bigg[ \sum_{h_v} p_v(h_v|e_v,\Theta_a',\Theta_b) - 1  \Bigg] 
\end{split}
\end{equation}

The Lagrangian is convex for all positive $c_v$ over the sets of defined constraints~\cite{Heskes2006}. The computation of partial derivatives of $L$ with respect to $p_u(h_u|\Theta_b)$ and $p_v(h_v|e_v,\Theta_a',\Theta_b)$ lead us to the expressions:

\begin{equation}\label{logqalpha}
\begin{split}
\log p_u(h_u|\Theta_b) &= - \nu_u - 1 - \sum_k \theta_k f_k(h_u) - \sum_{v \subset u} \lambda_{v \rightarrow u} (h_v) \\
\log p_v(h_v|e_v,\Theta_a',\Theta_b) &= - \frac{\nu_v}{c_v}- 1 +  \frac{1}{c_v} \sum_k g_k(h_v,e_v|\Theta_a') + \frac{1}{c_v}\sum_{v \subset u} \lambda_{v \rightarrow u} (h_v) 
\end{split}
\end{equation}

\noindent \noindent where $\theta_k f_k(h_u)$ refers herein to the feature function $f_k(h_u|\Theta_b)$ defined in Eq.(12) of the main paper. 

Now, we know that $\sum_{h_u} p_u(h_u|\Theta_b) = 1$, and similarly $p_v(h_v|e_v,\Theta_a',\Theta_b)$, therefore, computing the exponential values on both sides of the equation, and summing for $\sum_{h_u}$ and $\sum_{h_v}$, we obtain the values for $\nu_v$ and $\nu_u$ as: %
\begin{equation}\label{nuAlpha}
\begin{split}
\nu_u &= - 1 + \log \sum_{h_u} \exp \left\lbrace - \sum_k \theta_k f_k (h_u)  - \sum_{v \subset u} \lambda_{v \rightarrow u} (h_v)    \right\rbrace \\
\nu_v &= - c_v + c_v\log \sum_{h_v} \exp \left\lbrace  \frac{1}{c_v}\sum_k g_k(h_v,e_v|\Theta_a') +  \sum_{v \subset u} \frac{1}{c_v} \lambda_{v \rightarrow u} (h_v)  \right\rbrace
\end{split}
\end{equation}

The above are the corresponding partition functions for approximate marginals $p_u(h_u|\Theta_b)$ and conditional marginals $p_v(h_v|e_v,\Theta_a',\Theta_b)$. Plugging the expressions for $p_u(h_u|\Theta_b)$ and  $p_v(h_v|e_v,\Theta_a',\Theta_b)$ on the primal problem we find the dual problem $L^*$:  %

\begin{equation}\label{Dual}
L^*(\Lambda,\theta_b) = -\sum_u \log Z_u(\Theta_b,\Lambda) - \sum_{v} c_v\log Z_v(e_v,\Theta'_a,\Theta_b,\Lambda) - \sum_k \theta_k \mu_k
\end{equation}

\noindent which is convex~\cite{Bertsekas2003}. 
Algorithm~\ref{alg:MP} is the numerical implementation of block gradient descend method applied to the dual problem.
We find the optimal prior probabilities $p_u(h_u|\Theta_b)$ and a posterior probabilities $p_v(h_v|e_v,\Theta'_a,\Theta_b)$ by finding the optimal parameters $\Lambda$ and $\Theta$ minimizing the dual problem $L^*$, which is convex on $\lambda_{v\to u}(h_v)\in \Lambda$ and $\theta_k\in \Theta_b$. The partial derivative with respect to $\theta_k$ is: %
\begin{equation}
\frac{\partial L^* }{\partial \theta_k} = \sum_{h_u} f_k(h_u)p_u(h_u|\Theta_b) - \mu_k = 0
\end{equation}

Observe that the gradient of the dual problem with respect to $\theta_k$ is 0 when {\em moment-matching} constraints are satisfied. Since the prior $p_u(h_u|\Theta_b)$ depends on parameter $\theta_k$ we apply line search strategy to find better updates of $\theta_k$. We fix the length $\eta$ of each step according to the Armijo conditions and the update step is:%
\begin{equation}
\theta_k \leftarrow \theta_k' + \eta \left(\sum_{h_u} f_k(h_u) p_u(h_u|\Theta_b') - \mu_{k} \right )
\end{equation}

$\Theta_b$ are shared by all $\Psi_u$ and  we can perform several iterations before starting sending messages. In practice, this does not improve accuracy neither speed up the convergence. Therefore, we update once between each message-passing process.

The message-passing process consist of computing partial derivative with respect $\lambda_{v\to u}(h_v)$, being fixed model parameters $\Theta_b$ and $\Theta'_a$ found in the previous gradient descend and EM iteration, respectively. This block-gradient descend strategy has been successfully applied before in many other numerical schemes such as ~\cite{Heskes2006,Schwing2011}. Algorithm~\ref{alg:MP} follows the same ideas appearing in those papers. Partial derivative with respect to $\lambda_{v\to }(h_v)$ lead to the {\em sum-marginalization} constraint: 
\begin{equation}
\frac{\partial L^* }{\partial \lambda_{v \rightarrow u}(h_v)} = \sum_{h_{u \setminus v }} p_u( h_u|\Theta_b) - p_v(h_v|e_v,\Theta_a',\Theta_b) = 0
\end{equation}

For each hidden variable $h_v$, we fix $\Theta_b$, $\Theta'_a$ is always fixed in this algorithm, consider the partial derivative with respect to $\lambda_{v\to u }(h_v)$  and let us say that $\lambda_{v \to u }^{(new)}(h_v)$ is the value where the gradient  is 0. Then, from the {\em sum-marginalization consistency} constraint, we have: %
\begin{equation}\begin{split}
\frac{\partial L^*}{\partial \lambda_{v\to u }(h_v)  } &= \sum_{h_{u\setminus v}} p_{u}(h_u|\Theta_b) -p_{v}(h_v|e_v,\Theta_a',\Theta_b) =0\\
p_{v}(h_v|e_v,\Theta_a',\Theta_b) &= \frac{e^{\lambda_{v\to u }^{(new)}(h_v)}}{e^{\lambda_{v \to u}(h_v) } } p_{u}(h_v|\Theta_b)\\
\log p_{v}(h_v| e_v,\Theta_a',\Theta_b) &= \lambda_{v \to u }^{(new)}(h_v) - \lambda_{v \to u  }(h_v) + \\
&+\log p_{u}(h_v|\Theta_b) \\
\lambda_{v \to u  }^{(new)}(h_v) &= \lambda_{v \to u}(h_v) + \log p_{v}(h_v|e_v,\Theta_a',\Theta_b)  -\\
&- \log p_{u}(h_v|\Theta_b) \end{split}\label{eq:lambda_new}
\end{equation}

\noindent where we express the update message $\lambda_{v \to u }^{(new)}(h_v)$ in terms of the old messages. To estimate $p_{v}(h_v|e_v,\Theta'_a,\Theta_b)$ we add the last row of~\eqref{eq:lambda_new} over all  the pairs $u$ of hidden variables including $h_{v}$:%
\begin{equation}\begin{split}
\sum_{u \supset v }&\lambda_{v \to u}^{(new)}(h_v) = \sum_{u \supset v}\lambda_{v \to u }(h_v) + \\ 
+&\sum_{u\supset v}\log p_{v}(h_v|e_v,\Theta_a',\Theta_b) - \sum_{u\supset v}\log p_{u}(h_v|\Theta_b)
\end{split}
\end{equation}

Then, defining  $A_v$ as the number of pairs $u$ containing $v$ and rearranging the terms: %
\begin{equation}
\begin{split}
A_{v}\log p_{v}(h_v|e_v,\Theta'_a,\Theta_b) &= \sum_{u \supset v}\lambda_{v \to u }^{(new)}(h_v)  + \sum_{ u \supset v}[\log p_{u}(h_v|\Theta_b) - \lambda_{v \to u }(h_v)]\\
 c_{v}\log p_{v}(h_v|e_v,\Theta_a',\Theta_b) &=\nu_{v}-c_{v}  - \sum_k g_{k}(h_v,e_v|\Theta_a')  -\sum_{u \supset v} \lambda_{v\to u}^{(new)}(h_v)
\end{split}
\end{equation}

\noindent and adding both equations in both sides, we obtain the update for \\
$\log p_{v}(h_v|e_v,\Theta_a',\Theta_b) $: %
\begin{equation} \begin{split}
\log &\,p_{v}(h_v|e_v,\Theta_a',\Theta_b) = \frac{1}{c_{v}+A_{v}}\log \psi_{v}(h_v|e_v,\Theta'_a) + \\
+& \frac{1}{c_{v}+A_{v}}\sum_{u \supset v}\left[\log p_{u}(h_v|\Theta_b) - \lambda_{v \to u }(h_v)\right] 
\end{split}
\end{equation}

\noindent where $\log \psi_{v}(h_v|e_v,\Theta'_a) = \frac{\nu_{v}}{c_{v}} -1 -\frac{1}{c_{v}}\sum_k g_k(h_v,e_v|\Theta'_a)  $ and%
\begin{equation}
p_{v}(h_v|e_v,\Theta_a',\Theta_b) = \left(\psi_{v}(h_v|e_v)^{c_{v}}\prod_{u \supset v }\frac{p_{u}(h_v|\Theta_b)}{e^{\lambda_{v\to u }(h_v)}}\right )^{\frac{1}{c_{v}+A_{v}}}
\end{equation}

We can now to compute the update $\lambda_{v \to u }^{(new)}(h_v) $:%
\begin{equation}\begin{split}
\lambda_{v \to u }^{(new)}(h_v)&= \lambda_{v \to u }(h_v ) + \log p_{v}(h_v|e_v,\Theta_a',\Theta_b) - \\
&- \log p_{u}(h_v|\Theta_b)
\end{split}
\end{equation}

The above formulas are more compactly expressed if we define message functions by: $m_{v   \to u}(h_v)=e^{\lambda_{v \to u}(h_v)}$ and then, the update rules are:

\begin{equation}\begin{split}
m_{v \from u }(h_v) &=\frac{p_{u}(h_v|\Theta_b)}{m_{v \to u }(h_v)} \\
m_{v \to u}(h_v) & = \frac{p_{v}(h_v|e_v,\Theta_a',\Theta_b)}{m_{v \from u }(h_v)}
\end{split}
\end{equation}

In summary, we have the update expression to update the parameters of the feature functions, and the expression for the messages sent between the $u$ and $v$. These update formulas arranged as shown in Algorithm~\ref{alg:MP} provide a block gradient descend method that can be parallelized. At each iteration we find new updates of the prior model parameters $\Theta_b$, then we send  messages, first from hidden variables $h_v$ to pairs of hidden variables $u,$ and then we combine them obtaining the new messages to update the posteriori probability $p_v(h_v|e_v,\Theta'_a,\Theta_b)$ for the next EM iteration of the main algorithm.

\begin{algorithm}
\KwData{ $\{ \mu_k\}$: empirical moments.}
\KwResult{$\Theta_b$, $\Lambda$: model parameters, $\{p_v(h_v|e_v,\Theta_a',\Theta_b), p_u(h_u|\Theta_b)\}$ marginals.}
Initialize: $\theta_k= 0$, $\theta_k\in\Theta_b$, $m_{v \to u}(h_v)=1$;\\
\While{ not converged}
{

\For{ $\forall k \in \{I_u\}$ }{
$$\theta_k \leftarrow \theta_k' + \eta \left(\sum_{h_u} f_k(h_u) p_u(h_u|\Theta_b') - \mu_{k} \right )$$
}

\For{ $\forall v$  }
{
\For{ $\forall u \supset v$  }
{
$$ p_u(h_v|\Theta_b)=\sum_{h_{u \setminus v}}  p_u(h_u|\Theta_b) $$
$$m_{v \from u}(h_u)=\frac{p_u(h_v|\Theta_b)}{m_{v \to u}(h_v)}$$
}

$$p_v(h_v|e_v,\Theta_a',\Theta_b) = \frac{1}{Z_v} \left(e^{\sum_k g_k(h_v,e_v|\Theta_a') }\prod_{u \supset v} m_{v \from u}(h_u) \right )^{\!\!\frac{1}{c_v+A_v}}$$

\For{ $\forall u \supset v$  }
{
$$m_{v \to u}(h_v)   =  \frac{p_v(h_v|e_v,\Theta_a',\Theta_b)}{m_{v \from u}(h_u)}$$

$$p_u(h_u|\Theta_b) = \frac{1}{Z_u}\left(e^{ -\sum_k \theta_kf_k(h_u) }\prod_{v \subset u} m_{v \to u}(h_v)\right)^{\!\!\frac{1}{c_v}}$$
}
}
}
\caption{Message passing algorithm for constrained minimization of free energies. $Z_{u}$ and $Z_{v}$ are the partition function and $\eta$ is a step length satisfying the Armijo condition.}
\label{alg:MP}
\end{algorithm}




\end{document}